\documentclass[letterpaper, 10 pt, journal, twoside]{IEEEtran}
\usepackage{amsmath,amsfonts,amssymb,amsthm}
\newcommand\numberthis{\addtocounter{equation}{1}\tag{\theequation}}
\let\labelindent\relax
\usepackage{prettyref}
\usepackage{mathrsfs}
\usepackage{graphicx}
\usepackage{wrapfig}
\usepackage{subfig}
\usepackage{bbold}
\usepackage{tabu}
\usepackage{MnSymbol}
\usepackage{multirow}
\usepackage{booktabs}
\usepackage{enumitem}
\usepackage{algorithm}
\usepackage[table]{xcolor}
\usepackage[noend]{algpseudocode}
\usepackage
[backend=bibtex,
bibstyle=ieee,
citestyle=numeric,
mincitenames=1,
maxcitenames=2,
natbib=true,
doi=false,
isbn=false,
url=false,
eprint=false]{biblatex}
\usepackage[colorlinks,allcolors=gray,hypertexnames=true]{hyperref} 
\newcommand{\citewithauthor}[1]{\citeauthor{#1} \cite{#1}}

\newtheorem{theorem}{\TE{Theorem}}[section]
\newtheorem{lemma}[theorem]{\TE{Lemma}}
\newtheorem{remark}[theorem]{\TE{Remark}}
\newtheorem{assume}[theorem]{\TE{Assumption}}

\newtheorem{corollary}[theorem]{\TE{Corollary}}
\algnewcommand{\IfThenElse}[3]{
  \State \algorithmicif\ #1\ \algorithmicthen\ #2\ \algorithmicelse\ #3}
  
\algnewcommand{\LineComment}[1]{\State \(\triangleright\) #1}
\algdef{SE}[DOWHILE]{Do}{doWhile}{\algorithmicdo}[1]{\algorithmicwhile\ #1}
\newcommand*{\colorboxed}{}
\def\colorboxed#1#{%
  \colorboxedAux{#1}%
}
\newcommand*{\colorboxedAux}[3]{%
  \begingroup
    \colorlet{cb@saved}{.}%
    \color#1{#2}%
    \boxed{%
      \color{cb@saved}%
      #3%
    }%
  \endgroup
}

\newrefformat{fig}{Figure~\ref{#1}}
\newrefformat{par}{Section~\ref{#1}}
\newrefformat{appen}{Appendix~\ref{#1}}
\newrefformat{sec}{Section~\ref{#1}}
\newrefformat{sub}{Section~\ref{#1}}
\newrefformat{table}{Table~\ref{#1}}
\newrefformat{ass}{Assumption~\ref{#1}}
\newrefformat{alg}{Algorithm~\ref{#1}}
\newrefformat{def}{Definition~\ref{#1}}
\newrefformat{thm}{Theorem~\ref{#1}}
\newrefformat{cor}{Corollary~\ref{#1}}
\newrefformat{lem}{Lemma~\ref{#1}}
\newrefformat{step}{Step~\ref{#1}}
\newrefformat{ln}{Line~\ref{#1}}
\newrefformat{rem}{Remark~\ref{#1}}
\newrefformat{eq}{Equation~\ref{#1}}
\newrefformat{pb}{Problem~\ref{#1}}
\newrefformat{it}{Item~\ref{#1}}
\newrefformat{te}{Term~\ref{#1}}
\def\Eqref Eq:#1:{\eqref{eq:#1}}
\newrefformat{Eq}{Equation~\Eqref#1:}

\newcommand{\TE}[1]{\textbf{#1}}

\newcommand{\FPP}[2]{\frac{\partial{#1}}{\partial{#2}}}
\newcommand{\FPPR}[2]{{\partial{#1}}/{\partial{#2}}}
\newcommand{\FPPT}[2]{\frac{\partial^2{#1}}{\partial{#2}^2}}

\newcommand{\FOURR}[4]{\left(\setlength{\arraycolsep}{1pt}\begin{array}{cccc}{#1}^T, & {#2}^T, & {#3}^T, & {#4}^T\end{array}\right)^T}

\newcommand{\argmin}[1]{\underset{#1}{\text{argmin}}\;}

\newcommand{\ST}{\text{s.t.}}

\newcommand{\TWORCell}[2]{\begin{tabular}{@{}c@{}}#1 \\ #2\end{tabular}}

\newcommand{\ProjBox}[1]{\text{Proj}_{\{\ALLONE\geq v\geq \underline{v}\}}(#1)}
\newcommand{\Proj}[1]{\text{Proj}_X(#1)}
\newcommand{\nablaX}{\nabla_X}
\newcommand{\ALLONE}{\mathbb{1}}
\newcommand{\ALLZERO}{\mathbb{0}}

\usepackage{xcolor}
\definecolor{Blue} {rgb}{0.2, 0.2, 0.8}
\definecolor{Red}  {rgb}{0.8, 0.2, 0.2}
\definecolor{Green}{rgb}{0.2, 0.8, 0.2}
\makeatletter
\IEEEtriggercmd{\reset@font\normalfont\fontsize{5.0pt}{5.0pt}\selectfont}
\makeatother
\IEEEtriggeratref{1}

\makeatletter
\newcommand\fs@ruled@notop{\def\@fs@cfont{\bfseries}\let\@fs@capt\floatc@ruled
  \def\@fs@pre{}%
  \def\@fs@post{\kern2pt\hrule\relax}%
  \def\@fs@mid{\kern2pt\hrule\kern2pt}%
  \let\@fs@iftopcapt\iftrue}
\renewcommand\fst@algorithm{\fs@ruled@notop}
\makeatother

\newif\ifsupp
\supptrue
\title{\large\bf First-Order Bilevel Topology Optimization for Fast Mechanical Design}
\author{Zherong Pan, Xifeng Gao, and Kui Wu  \\
\thanks{Authors are with the Lightspeed \& Quantum Studio, Tencent-America. \{zrpan,xifgao,kwwu@tencent.com\}}}
        
\begin{document}
\maketitle
\allowdisplaybreaks
\thispagestyle{empty}
\pagestyle{empty}

\begin{abstract}
Topology Optimization (TO), which maximizes structural robustness under material weight constraints, is becoming an essential step for the automatic design of mechanical parts. However, existing TO algorithms use the Finite Element Analysis (FEA) that requires massive computational resources. We present a novel TO algorithm that incurs a much lower iterative cost. Unlike conventional methods that require exact inversions of large FEA system matrices at every iteration, we reformulate the problem as a bilevel optimization that can be solved using a first-order algorithm and only inverts the system matrix approximately. As a result, our method incurs a low iterative cost, and users can preview the TO results interactively for fast design updates. Theoretical convergence analysis and numerical experiments are conducted to verify our effectiveness. We further discuss extensions to use high-performance preconditioners and fine-grained parallelism on the Graphics Processing Unit (GPU).
\end{abstract}
\section{\label{sec:introduction}Introduction}
Several emerging techniques over the past decade are contributing to the new trend of lightweight, low-cost mechanical designs. Commodity-level additive manufacturing has been made possible using 3D printers and user-friendly modeling software, allowing a novice to prototype devices, consumer products, and low-end robots. While the appearance of a mechanical part is intuitive to perceive and modify, other essential properties, such as structural robustness, fragility, vibration frequencies, and manufacturing cost, are relatively hard to either visualize or manage. These properties are highly related to the ultimate performance criteria, including flexibility, robotic maneuverability, and safety. For example, the mass and inertia of a collaborative robot should be carefully controlled to avoid harm to the human when they collide. For manipulators on spaceships, the length of each part should be optimized to maximize the end-effector reachability under strict weight limits. Therefore, many mechanical designers rely on TO methods to search for the optimal structure of mechanical parts by tuning their center-of-mass \cite{wu2016shape}, inertia \cite{10.1145/2601097.2601157}, infill pattern \cite{WU2017358}, or material distribution \cite{10.1007/s00158-018-2143-8}.
\begin{figure}[th]
\centering
\includegraphics[width=\linewidth]{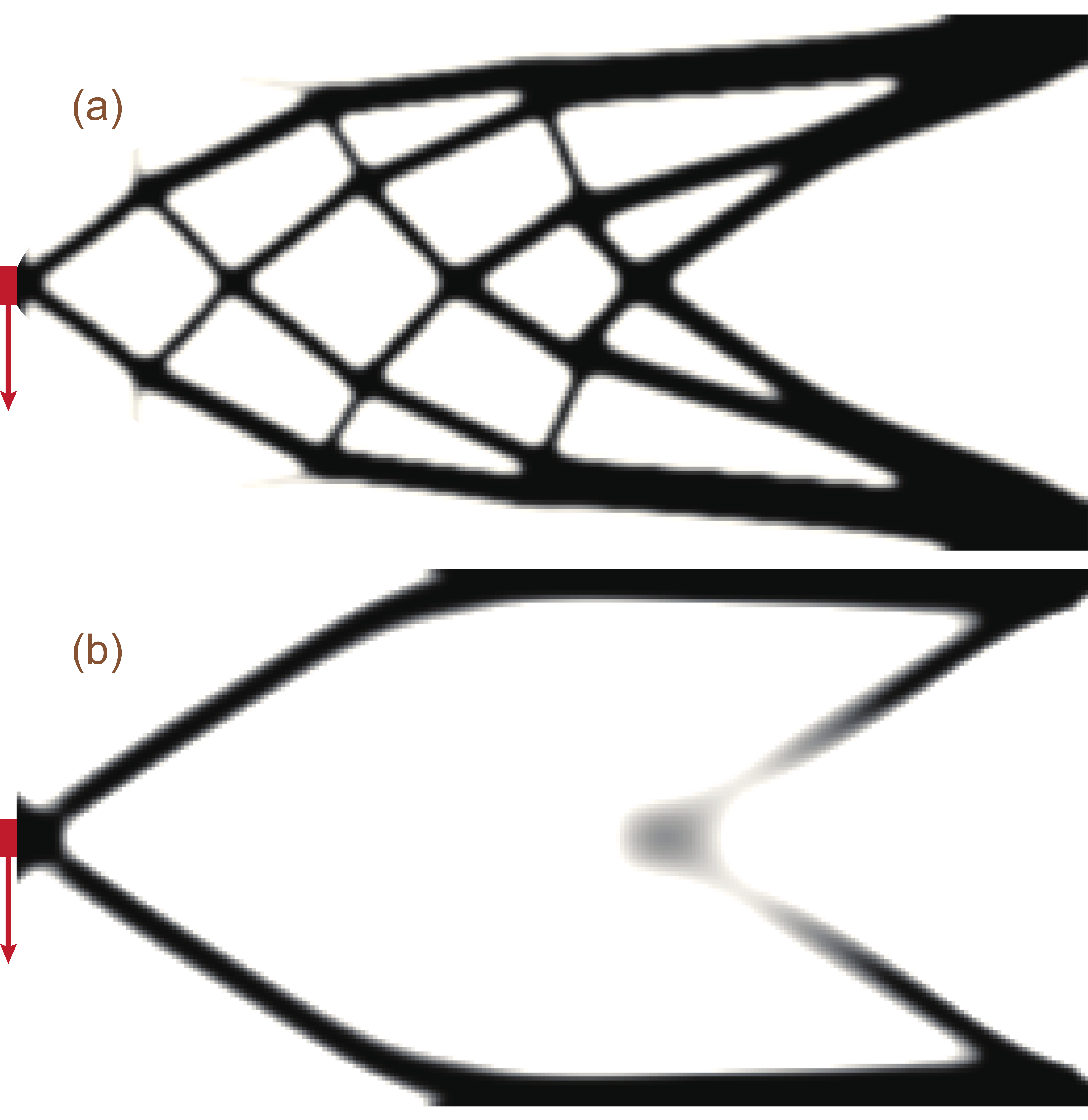}
\caption{\label{fig:teaser} \small{Given the same amount of computational time (200s) on a grid resolution $128\times256$, our method using approximate inverse (a) can find the complex mechanical designs to resist the external load (red force pointing downwards), while the prior method \cite{rojas2016efficient} (b) is still far from convergence.}}
\vspace{-20px}
\end{figure}

Unfortunately, the current TO algorithms \cite{sigmund2013topology} pose a major computational bottleneck. Indeed, TO algorithms rely on FEA to predict the relationship between the extrinsic (boundary loads) and intrinsic status (stress, strain, and energy densities) of a mechanical part. Such predictions are made by solving systems of (non-)linear equations during each iteration of TO. The design space and resulting system size are often gigantic when designing fine-grained material shapes. For example, an $128\times256$ design space illustrated in \prettyref{fig:teaser} incurs $10^4$ decision variables. Prior works make full use of high-end GPU \cite{MARTINEZFRUTOS201747} or many-core CPU \cite{mahdavi2006topology,aage2015topology} to solve such problems. Even after such aggressive optimization, the total time spent on FEA is typically at the level of hours. As a result, mechanical engineers cannot modify or preview their designs without waiting for tens of minutes or even hours, which significantly slow down the loops of design refinement.

\TE{Main Results:} We propose a low-cost algorithm that solves a large subclass of TO problems known as Solid Isotropic Material with Penalization (SIMP) \cite{sigmund200199}. The SIMP model is widely used to search for fine-grained geometric shapes by optimizing the infill levels. Our main observation is that, in conventional TO solvers, the computation resources are predominated by solving the system of equations to the machine precision. We argue that pursuing such precision for intermediary TO iterations is unnecessary as one only expects the system to be exactly solved on final convergence. Therefore, we propose to maintain an approximate solution that is refined over iterations. We show that such a TO solver corresponds to a bilevel reformulation of the SIMP problems (\prettyref{sec:problem}). We further propose a first-order algorithm to solve these bilevel problems with guaranteed convergence to a first-order critical point (\prettyref{sec:method}) of moderate accuracy. Compared with conventional TO algorithms \cite{sigmund2013topology}, our method incurs a much lower iterative cost of $\mathcal{O}(E\log E)$ ($E$ is the number of decision variables), allowing the mechanical designers to interactively preview the TO results and accelerate the loops of design refinement. We then discuss several useful extensions including different preconditioning schemes for refining the approximate solutions (\prettyref{sec:precon}) and a GPU-based implementation that achieves an asymptotic iterative cost of $\mathcal{O}(\log E)$ (\prettyref{sec:parallel}). We highlight the effectiveness of our method through a row of 2D benchmark problems from \cite{valdez2017topology,7332965,kambampati2020large} (\prettyref{sec:experiment}), where our iterative cost can be as low as $11-40$ms when handling a resolution of up to $110$k grid cells.

\section{\label{sec:related}Related Work}
We review representative TO formulations, numerical solvers for TO problems, and bilevel optimizations.

\TE{TO Formulations \& Applications:} Generally speaking, TO involves all mathematical programming problems that use geometric shapes as decision variables. A TO problem has three essential component: geometric representations, governing system of equations, and design objectives/constraints. The SIMP model \cite{sigmund200199} uses density-based representation where the geometric shape is discretized using FEA and decision variables are the infill levels of discrete elements. There are other popular representations. For example, truss optimization \cite{gomes2011truss} uses the decision variables to model the discrete beam existence. Infill pattern optimization \cite{WU2017358} searches for the local arrangements of shape interiors. Compared with the SIMP model, these formulations typically induce much fewer decision variables due to simplicity of shapes or locality. 

In terms of governing system of equations, the SIMP model assumes linearized elastic models, which means that the extrinsic and intrinsic properties of a mechanical part are linearly related. Linear models make fairly accurate predictions under the small-deformation assumption, which is the case with most mechanical design problems with rigid parts. For large deformations, there are extensions to the SIMP model that use nonlinear elastic models \cite{buhl2000stiffness,bruns2001topology} or even elastoplastic models \cite{maute1998adaptive}. However, the use of these nonlinear, non-convex models induces optimization problems of a more complex class than the SIMP model \cite{colson2005bilevel}, for which our analysis does not apply. In a similar fashion as SIMP, truss optimization \cite{SVANBERG198163} has small- and large-deformation variants. We speculate that similar ideas and analysis apply for truss optimization problems under small-deformations.

TO formulations deviate from each other mostly due to their design objectives and constraints. The SIMP model minimizes the total potential energy under a single external load and a total-volume constraint. Standard SIMP model does not have constraints to limit deformations, but there are extensions to include strain/stress constraints \cite{holmberg2013stress}. Unfortunately, stress-constrained formulations also induce a more complex problem class, where our analysis does not apply. Yet another extension to the SIMP model \cite{habbal2004multidisciplinary,holmberg2017game,dunning2011introducing} is to consider multiple or uncertain external loads. Such extension is favorable because a mechanical system can undergo a variety of unknown extrinsic situations. We speculate it is possible to adapt our method to these cases by sampling the uncertain situations and use our method as underlying solvers for each sampled scenario.

\TE{Numerical TO Solvers:} Three categories of numerical approaches have been extensively used: truncated gradient method, general-purpose method, and hierarchical method. The reference implementation \cite{sigmund200199} of the SIMP solver uses the heuristic, truncated gradient method \cite{bendsoe1995optimization}. This method first computes the exact gradient and then updates the design along the negative gradient direction using heuristic step sizes. Although the convergence of these heuristic schemes is hard to analyze, they perform reasonably well in practice and are thus widely used, e.g. in \cite{7332965}. General-purpose methods are off-the-shelf optimization algorithms (e.g., sequential quadratic programming (SQP)/interior point method \cite{rojas2015benchmarking}, and augmented Lagrangian method (ALM) \cite{da2018reliability}). The generality of these solvers makes it easier to experiment with different TO formulations. Without problem-specific optimization, however, they typically scale poorly to the high-dimensional decision space of practical problems. In particular, the (Globally Convergent) Method of Moving Asymptotes (GCMMA) \cite{https://doi.org/10.1002/nme.1620240207} is a special form of SQP that can handle a large set of box constraints, making it particularly suitable for SIMP-type problems. Finally, solving problem in a hierarchical approach is an attempting way to tackle high-dimensional problems. \citewithauthor{nguyen2017polytree} used a quadtree to focus computational resources on the material boundaries. The multigrid method \cite{7332965,amir2014multigrid} constructs a hierarchy of grids to accelerate FEA system solves.

\TE{Bilevel Optimization \cite{colson2005bilevel}:} When a constraint of an optimization problem (high-level problem) involves another optimization (low-level problem), the problem is considered bilevel. In our formulation, the high-level problem optimizes the material infill levels, while the low-level problem performs the FEA analysis. Bilevel problems frequently arise in robotics for time optimality \cite{tang2020enhancing}, trajectory search \cite{Pavone-RSS-19}, and task-and-motion planning \cite{toussaint2015logic}. As a special case considered in our method, if the low-level problem is strictly convex, then its solution is unique and the low-level problem can be replaced by a well-defined solution mapping function. This treatment essentially reduces the bilevel to a single-level optimization. Otherwise, low-level problem can be non-convex and the solution map is multi-valued, where additional mechanism is needed to compare and select low-level solutions (e.g., a maximum- or minimum-value function can be used and the associated problem is NP-hard \cite{bard1991some}). It is noteworthy that oftentimes an arbitrary solution of the low-level problem suffice, in which case penalty method \cite{mehra2021penalty} can be used to approximate one of the solutions on the low-level. Bilevel formulation has been used in TO community for years, but prior works resort to off-the-shelf solvers \cite{kovcvara1997topology} or machine learning \cite{NTOPO} to (approximately) solve the underlying optimization problem, while we utilize the special structure of the SIMP model to design our efficient, first-order method.
\section{\label{sec:problem}Problem Formulation}
In this section we introduce the SIMP model using notations in \cite{7332965}. \ifsupp(see \prettyref{table:symbols} in our appendix for a list of symbols) \fi This model uses FEA to discretize a mechanical part governed by the linear elastic constitutive law. If the material is undergoing an internal displacement $u(x): \Omega\to\mathbb{R}^2$ where $\Omega$ is the material domain, then the internal potential energy is accumulated according to the constitutive law as follows:
\begin{align*}
P[u(\bullet),v(\bullet)]=\int_\Omega\frac{1}{2}u(x)^Tk_e(v(x),x)u(x)dx,
\end{align*}
where $P[\bullet,\bullet]$ is the potential energy functional, $k_e$ is the spatially varying stiffness tensor parameterized by infinite-dimensional decision variable $v(x): \Omega\to\mathbb{R}$ that determines the infill levels. FEA discretizes the material domain $\Omega$ into a set of $E$ elements each spanning the sub-domain $\Omega_e$. The elements are connected by a set of $N$ nodes. By restricting $P[\bullet]$ to a certain finite-dimensional functional space (the shape space), the potential energy can be written as the following sum over elements:
\begin{align*}
P(u,v)=&\frac{1}{2}u^T\sum_{e=1}^KK_e(v)u\triangleq\frac{1}{2}u^TK(v)u\\
K_e(v)\triangleq&\int_{\Omega_e}k_e(v,x)dx.
\end{align*}
Here $K_e$ is the stiffness matrix of the $e$th element, $K$ is the stiffness matrix assembled from $K_e$. With a slight abuse of notation, we denote $P(\bullet,\bullet)$ as a finite-dimensional potential energy function, $u\in\mathbb{R}^{2N}$ without bracket as a vector of 2D nodal displacements ($u\in\mathbb{R}^{3N}$ in 3D), and $v\in\mathbb{R}^E$ without bracket as a vector of elementwise infill levels. The SIMP model is compatible with all kinds of FEA discretization scheme and the case with regular grid is illustrated in \prettyref{fig:pipeline}.
\begin{figure}[th]
\centering
\includegraphics[width=\linewidth]{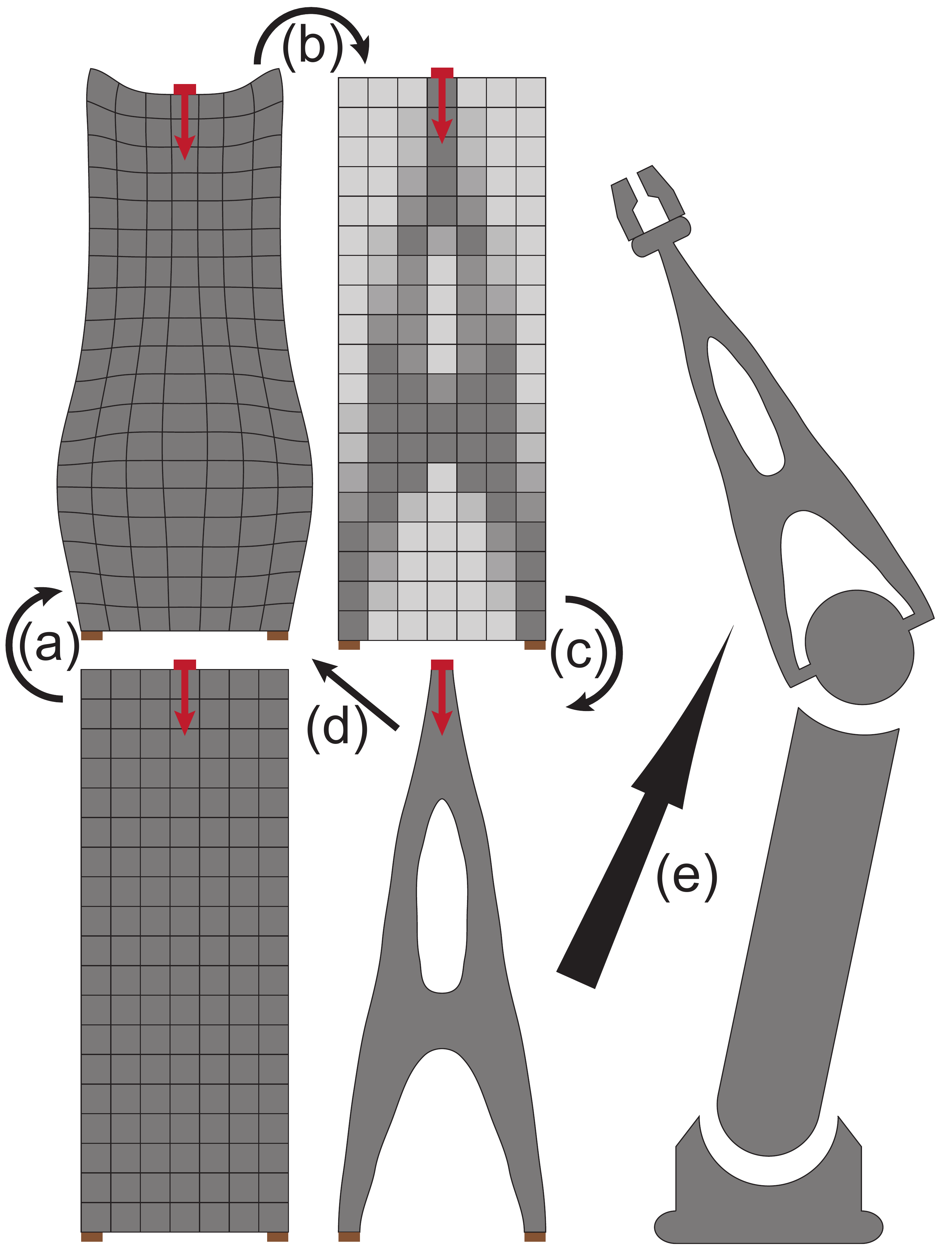}
\caption{\label{fig:pipeline} \small{Key steps of the SIMP model and the design loop: define the SIMP problem: external load (red) and fixture (brown); compute displacement $u$ using FEA (a); perform sensitivity analysis to compute $\nabla l$ (b); update infill levels (c); repeat (b,c,d) until convergence; deploy on a real robot (e).}}
\vspace{-5px}
\end{figure}

Suppose an external force $f$ is exerted on the mechanical part, then the total (internal+external) potential energy is:
\begin{align*}
P_f(u,v)=\frac{1}{2}u^TK(v)u-f^Tu.
\end{align*}
SIMP model works with arbitrary force setup, but external forces are only applied to the boundary in practice. For example, if only the $i$th boundary node is under a force $f_i$, then we have $f=e_if_i$ with $e_i$ being the unit vector. The equilibrium state is derived by minimizing $P_f$ with respect to $u$, giving $u_f=K^{-1}(v)f$ (\prettyref{fig:pipeline}a). As the name suggests the two variables $u_f,f$ are linearly related under the linear elastic constitutive law. The goal of TO is to minimize the induced internal potential energy due to $f$:
\begin{align}
\label{eq:opt}
\argmin{v\in X} l(v)\triangleq\frac{1}{2}u_f^T(v)K(v)u_f(v)=\frac{1}{2}f^TK^{-1}(v)f,
\end{align}
where $X$ is a compact, convex set bounding $v$ away from potentially singular configurations. Several widely used stiffness matrix parameterization of $K(v)$ includes the linear law $K(v)=\sum_{e=1}^Ev_eK_e$ or the power law $K(v)=\sum_{e=1}^Ev_e^\eta K_e$ where $\eta>0$ is some constant. In both cases, we need to choose $X$ such that the spectrum of $K$ is bounded from both sides:
\begin{align}
\label{eq:KVBound}
\ALLZERO<\underline{\rho}\ALLONE\leq\rho(K(v))\leq\bar{\rho}\ALLONE,
\end{align}
for any $v\in X$, where $\rho(\bullet)$ is the eigenvalues of a matrix. A typical choice of $X$ to this end is:
\begin{align*}
X=\{v|\ALLONE\geq v\geq\underline{v}\ALLONE>\ALLZERO\land\ALLONE^Tv\leq\bar{v}\},
\end{align*}
where $v_e$ is the infill level of the $e$th element, $\underline{v}$ is the minimal elementwise infill level, $\bar{v}$ is the total infill level, and $\ALLONE,\ALLZERO$ are the all-one and all-zero vectors, respectively. In this work, we use the following more general assumption on the stiffness matrix parameterization and the shape of $X$:
\begin{assume}
\label{ass:StiffnessParameter}
$X$ is a compact polytopic subset of $\mathbb{R}^E$, on which \prettyref{eq:KVBound} holds and $K(v)$ is smooth.
\end{assume}
\begin{remark}
\label{rem:modifier}
The compactness of $X$ is obvious. $X$ is also polytopic because we only have linear constraints, which is essential when we measure optimality using relative projection error (see \prettyref{sec:method} for details). \prettyref{ass:StiffnessParameter} is more general than the standard SIMP model. Indeed, $K(v)$ can be parameterized using any smooth activation function $\mathcal{A}_e(v)$ for the $e$th element as $K(v)=\sum_{e=1}^E\mathcal{A}_e(v)K_e$, which obviously satisfies \prettyref{ass:StiffnessParameter} as long as $\mathcal{A}_e(v)$ is bounded from both above and below on $X$, away from zero. Our assumption also allows a variety of topology modifications and constraints, some of which are illustrated here: \TE{Symmetry-Constraint:} We can plug in a left-right symmetric mapping matrix $\mathbb{S}$ and define $K(v)=K(\mathbb{S}v_s)$ where $v_s$ is the infill levels for the left-half of the material block. \TE{Component-Cost-Constraint:} A mechanical part can be divided into different sub-components and the amount of material can be assigned for each component as follows:
\begin{align*}
X\triangleq\{v|v=\FOURR{v^1}{v^2}{\cdots}{v^{\#}}\land \ALLONE^Tv^i\leq\bar{v}^i\},
\end{align*}
where $\#$ is the number of sub-components, $v^i$ is the sub-vector of $v$ consisting of decision variables for the $i$th component, and $\bar{v}^i$ is the total allowed amount of material for that sub-component. \TE{Material-Filtering:} \citewithauthor{guest2004achieving} proposed to control the thickness of materials by filtering the infill levels using some convolutional kernel denoted as $\mathbb{C}$. As long as the convolution operator $\mathbb{C}(\bullet)$ is smooth, \prettyref{ass:StiffnessParameter} holds by plugging in $K(v)=\sum_{e=1}^E\mathcal{A}_e(\mathbb{C}(v))K_e$.
\end{remark}

All the conventional methods such as the (GC)MMA \cite{svanberg2007mma,https://doi.org/10.1002/nme.1620240207}, SQP \cite{rojas2016efficient,rojas2015benchmarking}, and \cite{7332965} for solving \prettyref{eq:opt} involve computing the exact gradient via sensitivity analysis (\prettyref{fig:pipeline}b):
\small
\begin{align*}
\nabla l=&\frac{1}{2}\FPP{u_f^T}{v}Ku_f+\frac{1}{2}u_f^T\FPP{K}{v}u_f+\frac{1}{2}u_f^TK\FPP{u_f}{v}
=-\frac{1}{2}u_f^T\FPP{K}{v}u_f,
\end{align*}
\normalsize
where we have used the derivatives of the matrix inverse. The major bottleneck lies in the computation of $u_f$ that involves exact matrix inversion. Here we assume that $\FPPT{K}{v}$ is a third-order tensor of size $|u|\times|u|\times|v|$ and its left- or right-multiplication means contraction along the first and second dimension. Note that the convergence of GCMMA and SQP solvers rely on the line-search scheme that involves the computation of objective function values, which in turn requires matrix inversion. Practical methods \cite{7332965} and \cite{sigmund200199} avoid the line-search schemes using heuristic step sizes. Although working well in practice, the theoretical convergence of these heuristic rules is difficult to establish.

\section{\label{sec:method}Bilevel Topology Optimization}
In this section, we first propose our core framework to solve SIMP problems that allows inexact matrix inversions. We then discuss extensions to accelerate convergence via preconditioning (\prettyref{sec:precon}) and fast projection operators (\prettyref{sec:fastProjection}). Finally, we briefly discuss parallel algorithm implementations (\prettyref{sec:parallel}). Inspired by the recent advent of first-order bilevel optimization algorithms \cite{ghadimi2018approximation,hong2020two,khanduri2021near}, we propose to reformulation \prettyref{eq:opt} as the following bilevel optimization:
\begin{equation}
\begin{aligned}
\label{eq:optBilevel}
\argmin{v\in X}& f(u,v)\triangleq\frac{1}{2}u_f^TK(v)u_f\\
\ST&\quad u_f\in\argmin{u}P_f(u,v).
\end{aligned}
\end{equation}
The low-level part of \prettyref{eq:optBilevel} is a least-square problem in $u$ and, since $K(v)$ is always positive definite when $v\in X$, the low-level solution is unique. Plugging the low-level solution into the high-level objective function and \prettyref{eq:opt} is recovered. The first-order bilevel optimization solves \prettyref{eq:optBilevel} by time-integrating the discrete dynamics system as described in \prettyref{alg:Opt} and we denote this algorithm as \underline{F}irst-Order \underline{B}ilevel \underline{T}opology \underline{O}ptimization (FBTO). The first line of \prettyref{alg:Opt} is a single damped Jacobi iteration for refining the low-level solution using an adaptive step size of $\beta_k$. But instead of performing multiple Jacobi iterations until convergence, we go ahead to use the inexact result after one iteration for sensitivity analysis. Finally, we use a projected gradient descend to update the infill levels with an adaptive step size of $\alpha_k$.
\begin{algorithm}[ht]
\caption{\label{alg:Opt} FBTO}
\begin{algorithmic}[1]
\For{k$\gets1,2,\cdots$}
\State $u_{k+1}\gets u_k-\beta_k(K(v_k)u_k-f)$
\State $v_{k+1}\gets \Proj{v_k+\frac{\alpha_k}{2}u_k^T\FPP{K}{v_k}u_k}$
\label{ln:highlevel}
\EndFor
\end{algorithmic}
\end{algorithm}
Here $\Proj{\bullet}$ is the projection onto the convex set $X$ under Euclidean distance. This scheme has linear iterative cost of $\mathcal{O}(E)$ as compared with conventional method that has superlinear iterative cost due to sparse matrix inversions (some operations of \prettyref{alg:Opt} have a complexity of $\mathcal{O}(N)$ but we know that $E=\Theta(N)$).
\begin{remark}
The succinct form of our approximate gradient ($-u_k^T\FPPR{K}{v_k}u_k/2$ in \prettyref{ln:highlevel}) makes use of the special structure of the SIMP model, which is not possible for more general bilevel problems. If a general-purpose first-order bilevel solver is used, e.g. \cite{ghadimi2018approximation} and follow-up works, one would first compute/store $\nabla_{uv} P_f$ and then approximate $\nabla_{uu} P_f^{-1}\nabla_{uv} P_f$ via sampling. Although the stochastic approximation scheme does not require exact matrix inversion, they induce an increasing number of samples with a larger number of iterations. Instead, we make use of the cancellation between $K(v)$ in the high-level objective and $K^{-1}(v)$ in the low-level objective to avoid matrix inversion or its approximations, which allows our algorithm to scale to high dimensions without increasing sample complexity.
\end{remark}
As our main result, we show that \prettyref{alg:Opt} is convergent under the following choices of parameters.
\begin{assume}
\label{ass:BetaChoice}
We choose constant $\beta_k$ and some constant $p>0$ satisfying the following condition:
\begin{align}
\label{eq:ThetaBound}
0<\beta_k<\frac{2\underline{\rho}}{\bar{\rho}^2}\land
p>\frac{1-2\beta_k\underline{\rho}+\bar{\beta}^2\beta_k^2}{2\beta_k\underline{\rho}-\bar{\beta}^2\beta_k^2}.
\end{align}
\end{assume}
\begin{assume}
\label{ass:GammaChoice}
Define $\Delta_k\triangleq\|u_k-u_f(v_k)\|^2$ and suppose $\Delta_1\leq U^2$, we choose the following uniformly bounded sequence $\{\Gamma_k\}$ with constant $\bar{\Gamma}, \bar{\Theta}$:
\begin{equation*}
\bar{\Theta}\triangleq\frac{p+1}{p}(1-2\beta_k\underline{\rho}+\beta_k^2\bar{\rho}^2)\land\Gamma_k\leq\bar{\Gamma}\leq\frac{U^2(1-\bar{\Theta})}{(U+U_f)^4},
\end{equation*}
where $U_f$ is the finite upper bound of $u_f(v)$ on $X$.
\end{assume}
\begin{assume}
\label{ass:AlphaChoice}
We choose $\alpha_k$ as follows:
\begin{align*}
\alpha_k=\frac{1}{k^m}\sqrt{\frac{\bar{\Gamma}}{L_u^2L_{\nabla K}^2}\frac{4}{p+1}},
\end{align*}
for positive constants $m,p$, where $L_u$ is the L-constant of $u_f$ and $L_{\nabla K}$ is the upper bound of the Frobenius norm of $\FPPR{K(v)}{v}$ when $v\in X$.
\end{assume}
\begin{assume}
\label{ass:MNChoice}
We choose positive constants $m,n$ satisfying the following condition:
\begin{align*}
\max(1-2m+mn,2-2m-mn,2-3m)<0\land m<1.
\end{align*}
\end{assume}
\begin{theorem}
\label{thm:gradientConvergence}
Suppose $\{v_k\}$ is the sequence generated by \prettyref{alg:Opt} under \prettyref{ass:StiffnessParameter}, \ref{ass:BetaChoice}, \ref{ass:GammaChoice}, \ref{ass:AlphaChoice}, \ref{ass:MNChoice}. For any $\epsilon>0$, there exists an iteration number $k$ such that $\|\nablaX l(v_k)\|\leq\epsilon$, where $\nablaX l(v_k)$ is the projected negative gradient into the tangent cone of $X$, $\mathcal{T}_X$, defined as:
\begin{align*}
\nablaX l(v_k)=\argmin{d\in \mathcal{T}_X}\|d+\nabla l(v_k)\|.
\end{align*}
\end{theorem}
We prove \prettyref{thm:gradientConvergence} in \ifsupp\prettyref{sec:proof} \else our extended version \fi  where we further show that the high-level optimality error scales as $\mathcal{O}(k^{m-1})$, so that $m$ should be as small as possible for the best convergence rate. Direct calculation leads to the constraint of $m>3/4$ from the first two conditions of \prettyref{ass:MNChoice}, so the convergence rate can be arbitrarily close to $\mathcal{O}(k^{-1/4})$ as measured by the following relative projection error \cite{calamai1987projected}:
\begin{align*}
\Delta_k^v\triangleq&\frac{1}{\alpha_k^2}\|\Proj{v_k-\alpha_k\nabla l(v_k)}-v_k\|^2.
\end{align*}
Although our analysis of convergence uses a similar technique as that for the two-timescale method \cite{hong2020two}, our result is only single-timescale. Indeed, the low-level step size $\beta_k$ can be constant and users only need to tune a decaying step size for $\alpha_k$. We will further show that certain versions of our algorithm allow a large choice of $\beta_k=1$ (see \ifsupp\prettyref{sec:parameter}\else our extended version\fi). In addition, our formula for choosing $\alpha_k$ does not rely on the maximal number of iterations of the algorithm.

\begin{algorithm}[ht]
\caption{\label{alg:POpt} PFBTO}
\begin{algorithmic}[1]
\For{k$\gets1,2,\cdots$}
\State $\delta_k\gets M^{-1}(v_k)(K(v_k)u_k-f)$
\State $u_{k+1}\gets u_k-\beta_k K(v_k)M^{-1}(v_k)\delta_k$
\State $v_{k+1}\gets \Proj{v_k+\frac{\alpha_k}{2}u_k^T\FPP{K}{v_k}u_k}$
\EndFor
\end{algorithmic}
\end{algorithm}
\begin{algorithm}[ht]
\caption{\label{alg:CPOpt} CPFBTO}
\begin{algorithmic}[1]
\For{k$\gets1,2,\cdots$}
\State $u_{k+1}\gets u_k-\beta_k M^{-1}(v_k)(K(v_k)u_k-f)$
\label{ln:lowlevel}
\State $v_{k+1}\gets \Proj{v_k+\frac{\alpha_k}{2}u_k^T\FPP{K}{v_k}u_k}$
\EndFor
\end{algorithmic}
\end{algorithm}
\subsection{\label{sec:precon}Preconditioning}
Our \prettyref{alg:Opt} uses steepest gradient descend to solve the linear system in the low-level problem, which is known to have a slow convergence rate of $(1-1/\kappa)^{2k}/(1+1/\kappa)^{2k}$ \cite{luenberger1984linear} with $\kappa$ being the condition number of the linear system matrix. A well-developed technique to boost the convergence rate is preconditioning \cite{benzi2002preconditioning}, i.e., pre-multiplying a symmetric positive-definite matrix $M(v)$ that approximates $K(v)$ whose inverse can be computed at a low-cost. Preconditioning is a widely used technique in the TO community \cite{7332965,10.1145/3272127.3275012} to accelerate the convergence of iterative linear solvers in solving $u_f(v)=K(v)^{-1}f$. In this section, we show that FBTO can be extended to this setting by pre-multiplying $M(v)$ in the low-level problem and we name \prettyref{alg:POpt} as \underline{P}re-conditioned FBTO or PFBTO. \prettyref{alg:POpt} is solving the following different bilevel program from \prettyref{eq:optBilevel}:
\begin{equation}
\begin{aligned}
\label{eq:optBilevelSqr}
\argmin{v\in X}& f(u,v)\triangleq\frac{1}{2}u_f^TK(v)u_f\\
\ST&\quad u_f\in\argmin{u}\|K(v)u-f\|^2,
\end{aligned}
\end{equation}
where the only different is the low-level system matrix being squared to $K^2(v)$. Since $K(v)$ is non-singular, the two problems have the same solution set. \prettyref{eq:optBilevelSqr} allows us to multiply $M^{-1}(v)$ twice in \prettyref{alg:POpt} to get the symmetric form of $K(v)M^{-2}(v)K(v)$. To show the convergence of \prettyref{alg:POpt}, we only need the additional assumption on the uniform boundedness of the spectrum of $M(v)$:
\begin{assume}
\label{ass:BoundedMSpectrum}
For any $v\in X$, we have:
\begin{align*}
\ALLZERO<\underline{\rho}_M\ALLONE\leq\rho(M(v))\leq\bar{\rho}_M\ALLONE.
\end{align*}
\end{assume} 
\prettyref{ass:BoundedMSpectrum} holds for all the symmetric-definite preconditioners. The convergence can then be proved following the same reasoning as \prettyref{thm:gradientConvergence} with a minor modification summarized in \ifsupp\prettyref{sec:PreconditionedProof}\else our extended version\fi. Our analysis sheds light on the convergent behavior of prior works using highly efficient preconditioners such as geometric multigrid \cite{7332965,10.1145/3272127.3275012} and conjugate gradient method \cite{borrvall2001large}. We further extend these prior works by enabling convergent algorithms using lightweight preconditioners such as Jacobi/Gauss-Seidel iterations and approximate inverse schemes, to name just a few. The practical performance of \prettyref{alg:POpt} would highly depend on the design and implementation of specific preconditioners.

On the down side of \prettyref{alg:POpt}, the system matrix in the low-level problem of \prettyref{eq:optBilevelSqr} is squared and so is the condition number. Since the convergence speed of low-level problem is $(1-1/\kappa)^{2k}/(1+1/\kappa)^{2k}$, squaring the system matrix can significantly slow down the convergence, counteracting the acceleration brought by preconditioning. This is because we need to derive a symmetric operator of form $K(v)M^{-2}(v)K(v)$. As an important special case, however, only a single application of $M^{-1}(v)$ suffice if $M(v)$ commutes with $K(v)$, leading to the \underline{C}ommutable PFBTO or CPFBTO \prettyref{alg:CPOpt}. A useful commuting preconditioner is the Arnoldi process used by the GMRES solver \cite{saad1986gmres}, which is defined as:

\small
\begin{equation}
\begin{aligned}
\label{eq:LSP}
&M^{-1}(v)b\triangleq\sum_{i=0}^Dc_i^*K^i(v)b\\
&c_i^*\triangleq\argmin{c_i}\|b-\sum_{i=1}^{D+1}c_{i-1}K^i(v)b\|^2,
\end{aligned}
\end{equation}
\normalsize
where $D$ is the size of Krylov subspace. By sharing the same eigenvectors, $M(v)$ is clearly commuting with $K(v)$. This is a standard technique used as the inner loop of the GMRES solver, where the least square problem \prettyref{eq:LSP} is solved via the Arnoldi iteration. The Arnoldi process can be efficiently updated through iterations if $K(v)$ is fixed, which is not the case with our problem. Therefore, we choose the more numerically stable Householder QR factorization to solve the least square problem, and we name \prettyref{eq:LSP} as Krylov-preconditioner.

To show the convergence of \prettyref{alg:CPOpt}, we use a slightly different analysis summarized in the following theorem:
\begin{assume}
\label{ass:BetaChoicePre}
We choose constant $\beta_k$ and some constant $p>0$ satisfying the following condition:
\begin{align}
\label{eq:ThetaBoundPreconditioned}
0<\beta_k<\frac{2\underline{\rho}\underline{\rho}_M^2}{\bar{\rho}^2\bar{\rho}_M}\land
p>\frac{1-2\beta_k\frac{\underline{\rho}}{\bar{\rho}_M}+\beta_k^2\frac{\bar{\rho}^2}{\underline{\rho}_M^2}}{2\beta_k\frac{\underline{\rho}}{\bar{\rho}_M}-\beta_k^2\frac{\bar{\rho}^2}{\underline{\rho}_M^2}}.
\end{align}
\end{assume}
\begin{assume}
\label{ass:GammaChoicePre}
Define $\Xi_k\triangleq\|K(v_k)u_k-u_f\|^2$ and suppose $\Xi_1\leq U^2$, we choose the following uniformly bounded sequence $\{\Gamma_k\}$ with constant $\bar{\Gamma}, \bar{\Theta}$:
\begin{equation*}
\Theta_k\triangleq\frac{p+1}{p}(1-2\beta_k\frac{\underline{\rho}}{\bar{\rho}_M}+\beta_k^2\frac{\bar{\rho}^2}{\underline{\rho}_M})\land
\Gamma_k\leq\bar{\Gamma}\leq\frac{U^2(1-\bar{\Theta})}{(U/\underline{\rho}+U_f)^6},
\end{equation*}
where $U_f$ is the finite upper bound of $u_f(v)$ on $X$.
\end{assume}
\begin{assume}
\label{ass:AlphaChoicePre}
We choose $\alpha_k$ as follows:
\begin{align*}
\alpha_k=\frac{1}{k^m}\sqrt{\frac{\bar{\Gamma}}{L_K^2L_{\nabla K}^2}\frac{4}{p+1}},
\end{align*}
for some positive constants $m,p$, where $L_K$ is the L-coefficient of $K(v)$'s Frobenius norm on $X$.
\end{assume}
\begin{theorem}
\label{thm:gradientConvergencePre}
Suppose $M(v)$ commutes with $K(v)$, $\{v_k\}$ is the sequence generated by \prettyref{alg:CPOpt} under \prettyref{ass:StiffnessParameter}, \ref{ass:BetaChoicePre}, \ref{ass:GammaChoicePre}, \ref{ass:AlphaChoicePre}, \ref{ass:MNChoice}, \ref{ass:BoundedMSpectrum}. For any $\epsilon>0$, there exists an iteration number $k$ such that $\|\nablaX l(v_k)\|\leq\epsilon$, where $\nablaX l(v_k)$ is the projected negative gradient into the tangent cone of $X$.
\end{theorem}
The proof of \prettyref{thm:gradientConvergencePre} is provided in \ifsupp\prettyref{sec:PreconditionedProof}\else our extended version\fi, which allows the choice of a large, constant step size $\beta_k=1$ when the Krylov-preconditioner is used.

\subsection{\label{sec:fastProjection}Efficient Implementation}
We present some practical strategies to further accelerate the computational efficacy of all three \prettyref{alg:Opt},\ref{alg:POpt},\ref{alg:CPOpt} that are compatible with our theoretical analysis. First, the total volume constraint is almost always active as noted in \cite{svanberg2007mma}. Therefore, it is useful to maintain the total volume constraint when computing the approximate gradient. If we define the approximate gradient as $\tilde{\nabla}l\triangleq-\frac{1}{2}u_k^T\FPPR{K}{v_k}u_k$ and consider the total volume constraint $\ALLONE^Tv_k=\bar{v}$, then a projected gradient can be computed by solving:
\begin{align*}
\tilde{\nabla}_Pl\triangleq\argmin{\ALLONE^T\tilde{\nabla}_Pl=0}\|\tilde{\nabla}_Pl-\tilde{\nabla}l\|^2,
\end{align*}
with the analytic solution being $\tilde{\nabla}_Pl=\tilde{\nabla}l-\ALLONE\ALLONE^T\tilde{\nabla}l/E$. In our experiments, using $\tilde{\nabla}_Pl$ in place of $\tilde{\nabla}l$ in the last line of FBTO algorithms can boost the convergence rate at an early stage of optimization. 

Second, the implementation of the projection operator $\Proj{\bullet}$ can be costly in high-dimensional cases. However, our convex set $X$ typically takes a special form that consists of mostly bound constraints $\ALLONE\geq v\geq \underline{v}\ALLONE$ with only one summation constraint $\ALLONE^Tv\leq\bar{v}$. Such a special convex set is known as a simplex and a special $\mathcal{O}(E\log(E))$ algorithm exists for implementing the $\Proj{\bullet}$ as proposed in \cite{condat2016fast}. Although the original algorithm \cite{condat2016fast} only considers the equality constraint $\ALLONE^Tv=1$ and single-sided bounds $v\geq0$, a similar technique can be adopted to handle our two-sided bounds and inequality summation constraint and we provide its derivation for completeness. We begin by checking whether the equality constraint is active. We can immediately return if $\ALLONE^T\ProjBox{v_k}\leq\bar{v}$. Otherwise, the projection operator amounts to solving the following quadratic program:
\begin{align*}
\argmin{\ALLONE\geq v\geq \underline{v}\ALLONE}\frac{1}{2}\|v-v_k\|^2\quad\ST\;\ALLONE^Tv=\bar{v},
\end{align*}
whose Lagrangian $\mathcal{L}$ and first-order optimality conditions are:
\small
\begin{align*}
\mathcal{L}=&\frac{1}{2}\|v-v_k\|^2+\lambda_\ALLONE^T(\ALLONE-v)+\lambda_{\underline{v}}^T(v-\underline{v}\ALLONE)+\lambda(\ALLONE^Tv-\bar{v})\\
v=&v_k+\lambda_\ALLONE-\lambda_{\underline{v}}-\lambda\ALLONE\land
\ALLZERO\leq\lambda_\ALLONE\perp\ALLONE-v\geq\ALLZERO\land
\ALLZERO\leq\lambda_{\underline{v}}\perp{v-\underline{v}\ALLONE}\geq\ALLZERO,
\end{align*}
\normalsize
where $\lambda_\ALLONE,\lambda_{\underline{v}},\lambda$ are Lagrange multipliers. We can conclude that $v=\text{Clamp}(v_k-\lambda\ALLONE,\underline{v}\ALLONE,\ALLONE)$ and the following equality holds due to the constraint $\ALLONE^Tv=\bar{v}$ being active:
\begin{align}
\label{eq:piecewiseLinearEQ}
\bar{v}=\ALLONE^T\text{Clamp}(v_k-\lambda\ALLONE,\underline{v}\ALLONE,\ALLONE),
\end{align}
which is a piecewise linear equation having at most $2E$ pieces. There are $E$ left-nodes in the form of $v_k-\underline{v}\ALLONE$ and $E$ right-nodes in the form of $v_k-\ALLONE$, separating different linear pieces. The piecewise linear equation can be solved for $\lambda$ and thus $v$ by first sorting the $2E$ nodes at the cost of $\mathcal{O}(E\log E)$ and then looking at each piece for the solution. We summarize this process in \prettyref{alg:Proj} where we maintain the sorted end points of the line segments via running sums.
\begin{algorithm}[ht]
\caption{\label{alg:Proj} Project $v_k$ onto $X$ ($\text{SLOPE}$ is the slope of each line segment, $\text{RHS}$ is the righthand side of \prettyref{eq:piecewiseLinearEQ})}
\begin{algorithmic}[1]
\If{$\ALLONE^T\ProjBox{v_k}\leq\bar{v}$}
\State Return $\ProjBox{v_k}$
\EndIf
\State $\text{ARRAY}\gets\emptyset$
\For{$v_k^e\in v_k$}
\State $\text{ARRAY}\gets\text{ARRAY}\bigcup \{v_k^e-\underline{v},v_k^e-1\}$
\EndFor
\State Sort $\text{ARRAY}$ from low to high\Comment{$\mathcal{O}(E\log E)$}
\State $\text{SLOPE}\gets0$, $\text{RHS}\gets E$, $\text{NODES}\gets\emptyset$, $\lambda_\text{last}^*\gets0$
\For{$\lambda^*\in\text{ARRAY}$}\Comment{$\mathcal{O}(E)$}
\State $\text{RHS}\gets\text{RHS}+\text{SLOPE}(\lambda^*-\lambda_\text{last}^*)$
\If{$\lambda^*$ of form $v_k^e-1$}
\State $\text{SLOPE}\gets\text{SLOPE}-1$
\EndIf
\State \algorithmicelse\ $\text{SLOPE}\gets\text{SLOPE}+1$
\Comment{$\lambda^*$ of form $v_k^e-\underline{v}$}
\State $\text{NODES}\gets\text{NODES}\bigcup\{<\lambda^*,\text{RHS}>\}$, $\lambda_\text{last}^*\gets\lambda^*$
\EndFor
\For{Consecutive $\text{NODE}, \text{NODE}'\in\text{NODES}$}\Comment{$\mathcal{O}(E)$}
\If{Segment $<\text{NODE}, \text{NODE}'>$ passes through $\bar{v}$}
\State Solve for $\lambda$ and return $\text{Clamp}(v_k-\lambda\ALLONE,\underline{v}\ALLONE,\ALLONE)$
\EndIf
\EndFor
\end{algorithmic}
\end{algorithm}
\begin{figure*}[t]
\centering
\includegraphics[width=\linewidth]{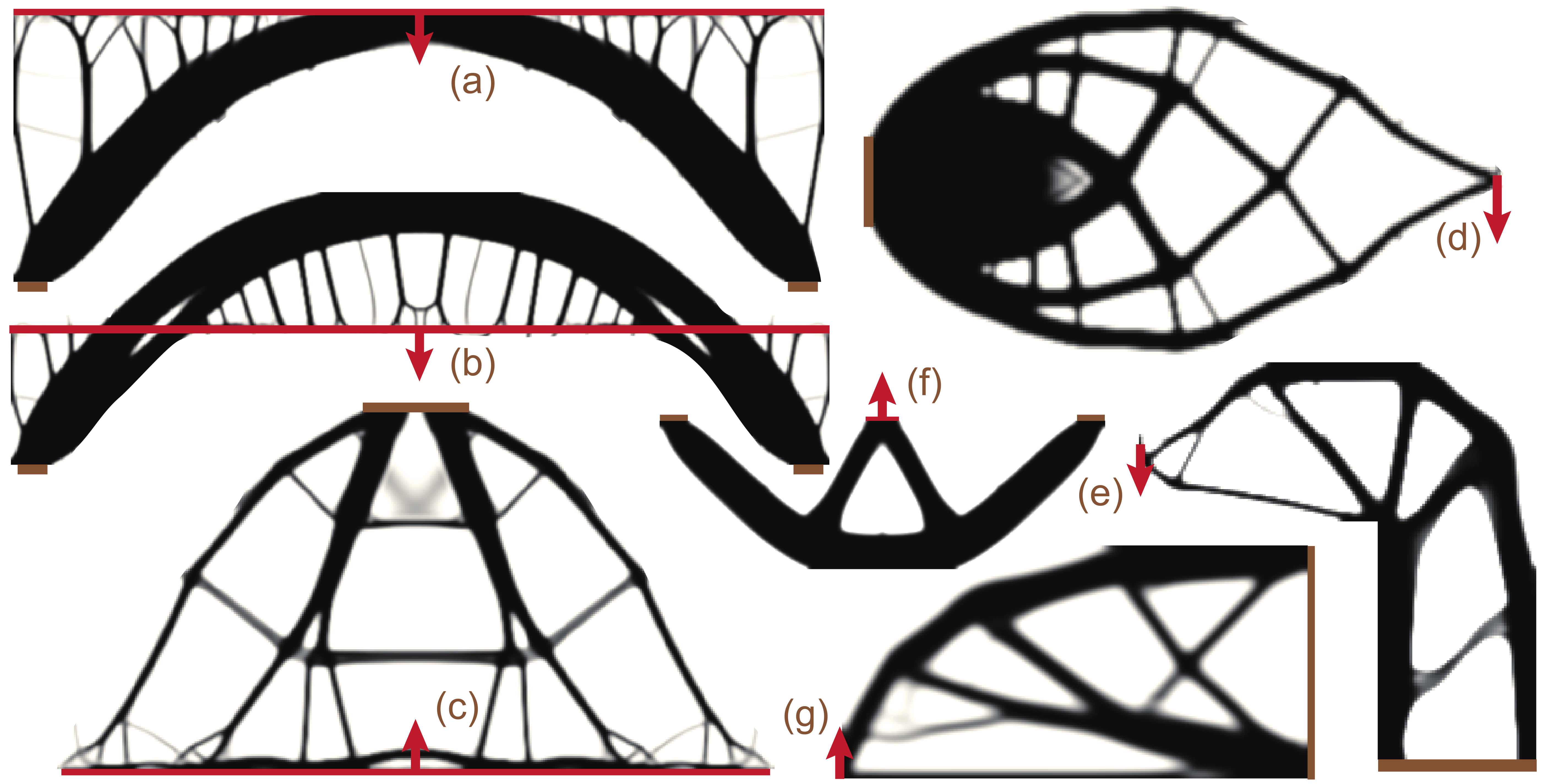}
\caption{\label{fig:gallery} \small{A gallery of benchmark problems selected from \cite{valdez2017topology,7332965} and 2D version of problems in \cite{kambampati2020large}, which are solved using \prettyref{alg:CPOpt}.}}
\vspace{-5px}
\end{figure*}
\subsection{\label{sec:parallel}Fine-Grained Parallelism}
\begin{figure}[th]
\centering
\includegraphics[width=\linewidth]{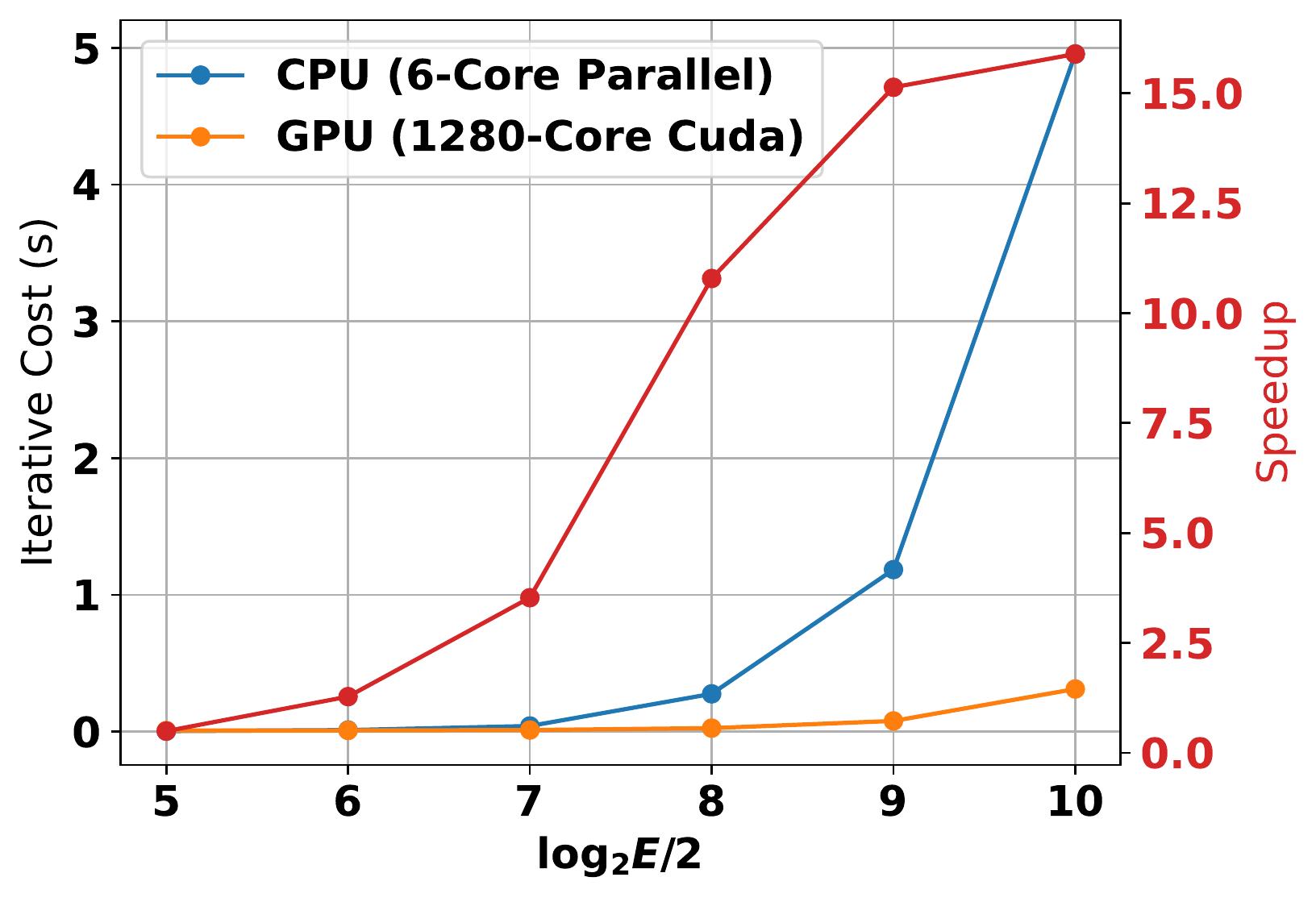}
\vspace{-15px}
\caption{\label{fig:parallel} \small{The acceleration rate of low-end mobile GPU ($1280$ cores) versus CPU (6 cores) implementation under different resolutions of discretization. We achieve a maximal speedup of $16$ times.}}
\vspace{-5px}
\end{figure}
Prior work \cite{7332965} proposes to accelerate TO solver on GPU, but they require a complex GPU multgrid implementation which is also costly to compute per iteration. In comparison, our \prettyref{alg:Opt},\ref{alg:POpt},\ref{alg:CPOpt} can make full use of many-core hardwares with slight modifications. In this section, we discuss necessary modifications for a GPU implementation assuming the availability of basic parallel scan, reduction, inner-product, and radix sort operations \cite{sanders2010cuda}. A bottleneck in \prettyref{alg:Opt},\ref{alg:POpt},\ref{alg:CPOpt} is matrix-vector production $K(v)u$, of which the implementation depends on the type of FEA discretization. If the discretization uses a regular pattern, such as a regular grid in \prettyref{fig:pipeline}, then the element-to-node mapping is known and can be hard coded, so the computation for each node is independent and costs $\mathcal{O}(1)$ if $E$ thread blocks are available. For irregular discretizations, we have to store the sparse matrix explicitly and suggest using a parallel sparse matrix-vector product routine \cite{bell2008efficient}. To implement the Krylov-preconditioner, we perform in-place QR factorization to solve for $c_i^*$. The in-place QR factorization involves computing the inner-product for $D^2$ times and then solving the upper-triangular system. Altogether, the least square solve in \prettyref{eq:LSP} costs $\mathcal{O}(D^2\log E+D^2)$ when $E$ threads are available. Here we use a serial implementation of the upper-triangular solve, which is not a performance penalty when $D\ll E$. Finally, \prettyref{alg:Proj} involves one radix sort that costs $\mathcal{O}(\log E)$ and three for loops, the first and last for loops do not have data dependency and cost $\mathcal{O}(1)$. The second for loop accumulates two variables ($\text{Slope}, \text{Total}$) that can be accomplished by a parallel scan taking $\mathcal{O}(\log E)$. In summary, the cost of each iteration of \prettyref{alg:Opt},\ref{alg:POpt},\ref{alg:CPOpt} can be reduced to $\mathcal{O}(D^2\log E)$ when $\mathcal{O}(E)$ threads are available. We profile the parallel acceleration rate in \prettyref{fig:parallel}.
\begin{figure}[ht]
\vspace{-5px}
\centering
\setlength{\tabcolsep}{5pt}
\begin{tabular}{l}
\includegraphics[width=\linewidth]{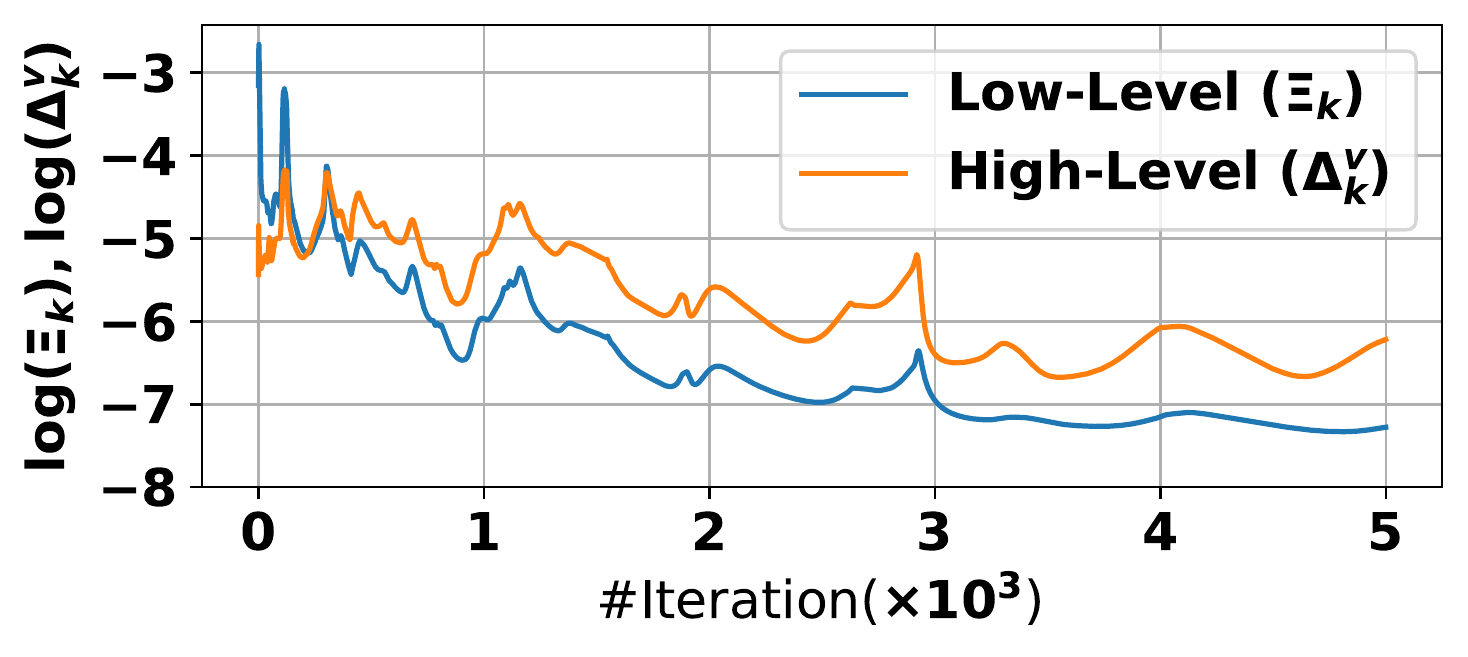}\vspace{-5px}
\put(-200,7){(a)}\\
\includegraphics[width=\linewidth]{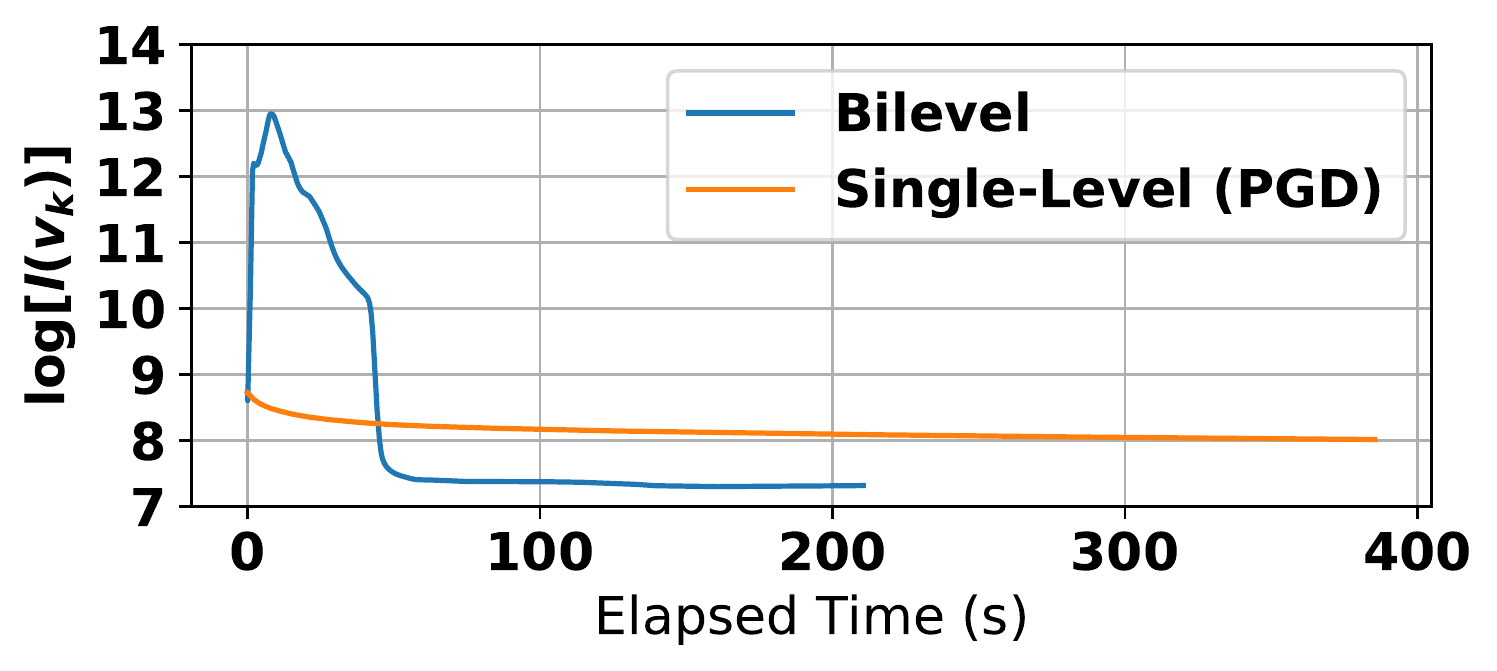}\vspace{-5px}
\put(-200,7){(b)}
\end{tabular}
\caption{\label{fig:history} \small{For the benchmark problem of \prettyref{fig:teaser}, we run bilevel optimization (\prettyref{alg:CPOpt}) for $5000$ iterations and single-level optimization (PGD) for $200$ iterations. (a): We profile our low- and high-level error evolution, plotted against the number of iterations. (b): We compare the convergence history of true energy function value $l(v_k)$ of both algorithms, plotted against computational time. Our method not only converges faster but also converges to a better solution.}}
\vspace{-5px}
\end{figure}
\begin{figure}[ht]
\vspace{-5px}
\centering
\includegraphics[width=\linewidth]{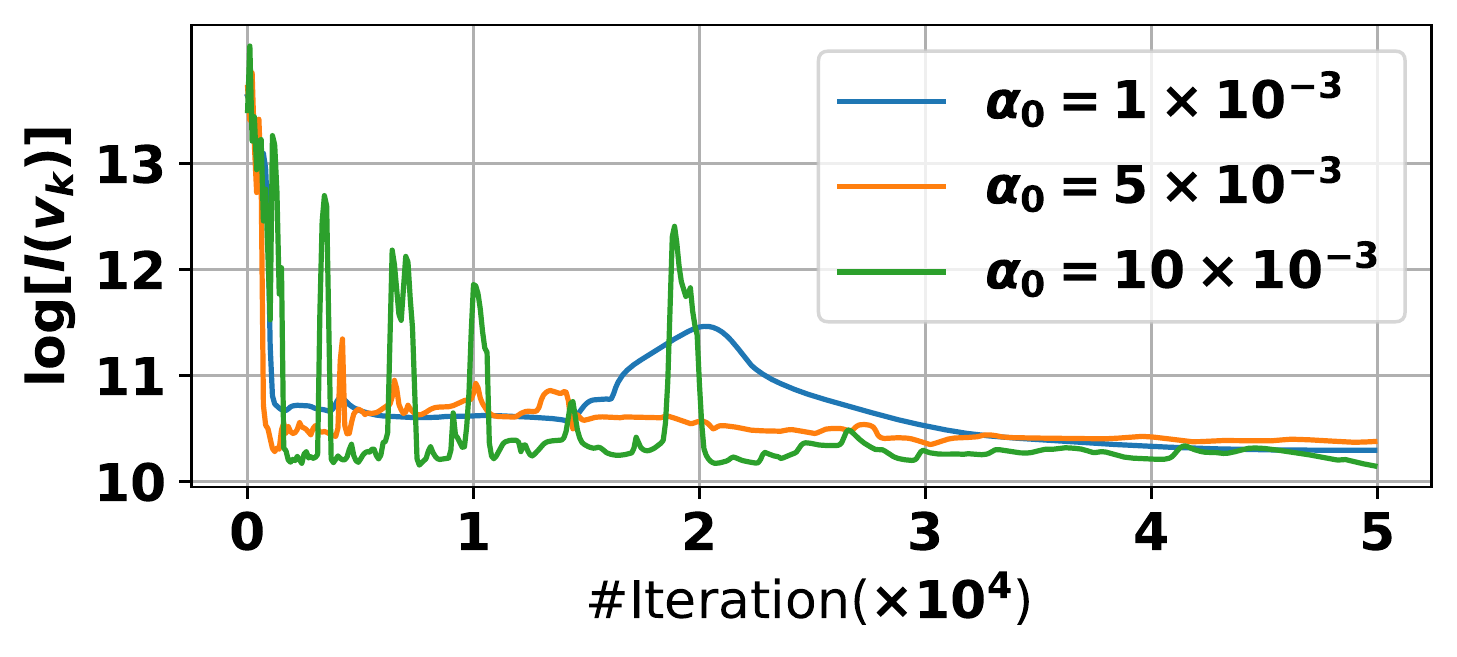}
\caption{\label{fig:alpha0} \small{For the benchmark problem in \prettyref{fig:gallery}e, we compared the convergence history of three choices of $\alpha_0$. A larger $\alpha_0$ would lead to more fluctuations at an early stage, but all three optimizations would ultimately converge. When $\alpha_0>10^{-2}$, we observe divergence.}}
\vspace{-5px}
\end{figure}
\begin{figure}[ht]
\vspace{-5px}
\centering
\includegraphics[width=\linewidth]{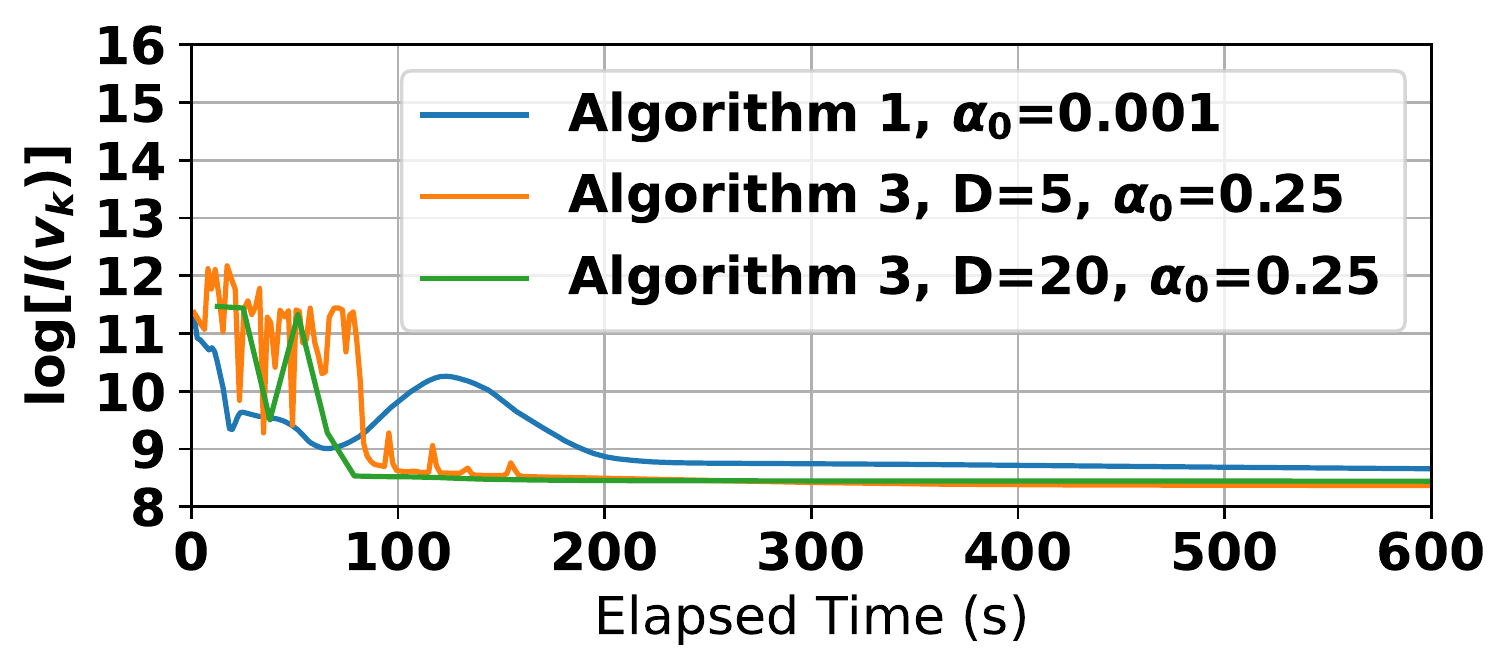}
\caption{\label{fig:DHistory} \small{We compare the convergence history of \prettyref{alg:Opt} and \prettyref{alg:CPOpt} under different $D$, in terms of exact energy value plotted against computational time. We need to use an extremely small $\alpha_0=0.001$ for \prettyref{alg:Opt} to converge. For \prettyref{alg:CPOpt} with $D\geq5$, a much larger $\alpha_0=0.25$ can be used, although there are some initial fluctuations when $D=5$. \prettyref{alg:CPOpt} with $D=20$ exhibits the highest overall convergence speed.}}
\vspace{-5px}
\end{figure}
\begin{figure}[ht]
\vspace{-5px}
\centering
\includegraphics[width=\linewidth]{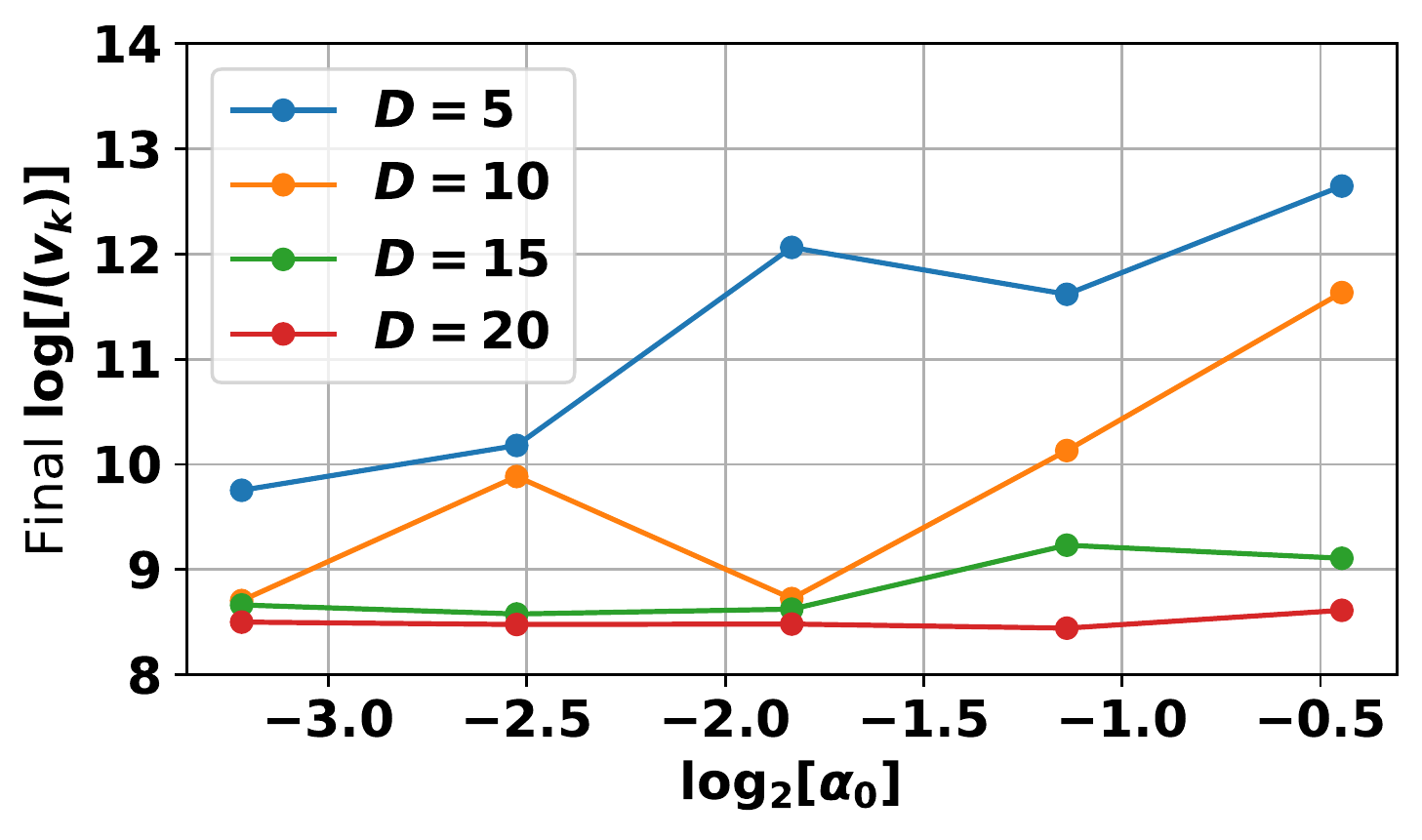}
\caption{\label{fig:D} \small{We run our method for $5000$ iterations and plot the final exact energy value against $\alpha_0$ where $\log(l(v_k))>9$ indicates divergence. When $D\geq15$, our method is convergent under all the step sizes. Therefore, we choose $D=20$ in all examples and only leave $\alpha_0$ to be tuned by users.}}
\vspace{-5px}
\end{figure}
\section{\label{sec:experiment}Experiments}
In this section, we compare our most efficient \prettyref{alg:CPOpt} with other methods that can provide convergence guarantee. We implement our method using native C++ on a laptop machine with a 2.5G 6-core Intel i7 CPU and a 1.48G Nvidia GTX 1060 mobile GPU having 1280 cores. Our implementation makes full use of the cores on CPU via OpenMP and GPU via Cuda and we implement the regular grid FEA discretization. In all the examples, we apply a $7\times7$ separable Gaussian kernel $\mathbb{C}$ as our material filter (see \prettyref{rem:modifier}), which can be implemented efficiently on GPU as a product of multiple linear kernels. We found that a $7\times7$ kernel achieves the best balance between material thickness and boundary sharpness. Unless otherwise stated, we choose \prettyref{alg:CPOpt} with $\underline{v}=0.1, \eta=3, \alpha_k=\alpha_0k^{-3/4}, \beta_k=1, D=20$ for our Krylov-preconditioner in all the experiments where $\alpha_0$ varies across different experiments. \ifsupp We justify these parameter choices in \prettyref{sec:parameter}. \fi For the three algorithms in \prettyref{sec:method}, we terminate when the relative change to infill levels over two iterations ($\|v_{k+1}-v_k\|_\infty$) is smaller than $10^{-4}$ and the absolute error of the linear system solve ($\|K(v_k)u_k-f\|_\infty$) is smaller than $10^{-2}$, which can be efficiently checked on GPU by using a reduction operator that costs $\mathcal{O}(\log E)$. The two conditions are measuring the optimality of the high- and low-level problems, respectively. The later is otherwise known as the tracking error, which measures how closely the low-level solution $u_k$ follows its optima as the high-level solution $v_k$ changes. We choose 7 benchmark problems from \cite{valdez2017topology,7332965,kambampati2020large} to profile our method as illustrated in \prettyref{fig:gallery} and summarized in \prettyref{table:data}. Our iterative cost is at most $40$ms when handling a grid resolution of $192\times576$, which allows mechanical designers to quickly preview the results.
\begin{table}[ht]
\begin{center}
\resizebox{.48\textwidth}{!}{
\rowcolors{0}{gray!50}{white}
\begin{tabular}{lcccccc}
\toprule
Benchmark &  Res. & Frac. & \#Iter. & \TWORCell{Time}{(Total)} & \TWORCell{Time}{(Iter.)} & Mem. \\
\midrule
\prettyref{fig:teaser}   & $128\times256$ & $0.4$ & $49600$ & $694$ & $0.014$ & $8$\\
\prettyref{fig:gallery}a & $192\times576$ & $0.3$ & $14600$ & $584$ & $0.040$ & $24$\\
\prettyref{fig:gallery}b & $192\times576$ & $0.3$ & $8900$ & $356$ & $0.040$ & $24$\\
\prettyref{fig:gallery}c & $192\times384$ & $0.3$ & $37400$ & $1159$ & $0.031$ & $18$\\
\prettyref{fig:gallery}d & $160\times240$ & $0.4$ & $20400$ & $346$ & $0.017$ & $8$\\
\prettyref{fig:gallery}e & $160\times160$ & $0.5$ & $13300$ & $146$ & $0.011$ & $4$\\
\prettyref{fig:gallery}f & $128\times384$ & $0.5$ & $9300$ & $195$ & $0.021$ & $10$\\
\prettyref{fig:gallery}g & $128\times256$ & $0.3$ & $19400$ & $252$ & $0.013$ & $8$\\
\bottomrule
\end{tabular}}
\end{center}
\caption{\label{table:data} \small{Statistics of benchmark problems: benchmark index, grid resolution, volume fraction $(\bar{v})$, iterations until convergence, total computational time(s), iterative cost(s), and GPU memory cost (mb).}}
\vspace{-5px}
\end{table}

\paragraph{Convergence Analysis} We highlight the benefit of using approximate matrix inversions by comparing our method with Projected Gradient Descent (PGD). Using a single-level formulation, PGD differs from \prettyref{alg:CPOpt} by using exact matrix inversions, i.e., replacing \prettyref{ln:lowlevel} with $u_{k+1}\gets K^{-1}(v_k)f$. We profile the empirical convergence rate of the two formulations in \prettyref{fig:history} for solving the benchmark problem in \prettyref{fig:teaser}. The comparisons on other benchmark problems are given in \ifsupp\prettyref{sec:additionalExpr}. \else our extended version. \fi Since exact, sparse matrix factorization prevents PGD from being implemented on GPU efficiently, our method allows fine-grained parallelism and is much more GPU-friendly. For comparison, we implement other steps of PGD on GPU while revert to CPU for sparse matrix factorization. As shown in \prettyref{fig:history}a, the high- and low-level error in our method suffers from some local fluctuation but they are both converging overall. We further plot the evolution of exact energy value $l(v_k)$. After about $70$s of computation, our method converges to a better optima, while PGD is still faraway from convergence as shown in \prettyref{fig:history}b. We further notice that our method would significantly increase the objective function value at an early stage of optimization. This is because first-order methods rely on adding proximal regularization terms to the objective function to form a Lyapunov function (see e.g. \cite{chen2021closing}). Although our analysis is not based on Lyapunov candidates, we speculate a similar argument to \cite{chen2021closing} applies.

\paragraph{Parameter Sensitivity} When setting up the SIMP problem, the user needs to figure out the material model $K_e$ and external load profile $f$. We use linear elastic material relying on Young's modulus and Poisson's ratio. We always set Young's modulus to $1$ and normalize the external force $f$, since these two parameters only scale the objective function without changing the solutions. When solving the SIMP problem, our algorithm relies on two critical parameters: initial high-level step size $\alpha_0$ and Krylov subspace size $D$ for preconditioning. We profile the sensitivity with respect to $\alpha_0$ in \prettyref{fig:alpha0} for benchmark problem in \prettyref{fig:gallery}e. We find our method convergent over a wide range of $\alpha_0$, although an overly large $\alpha_0$ could lead to excessive initial fluctuation. In practice, matrix-vector product is the costliest step and takes up $96\%$ of the computational time, so the iterative cost is almost linear in $D$. As illustrated in \prettyref{fig:DHistory}, \prettyref{alg:Opt} would require an extremely small $\alpha_0$ and more iterations to converge, while \prettyref{alg:CPOpt} even with a small $D$ can effectively reduce the low-level error and allow a wide range of choices for $\alpha_0$, leading to a faster overall convergence speed. We profiled the sensitivity to subspace size $D$ in \prettyref{fig:D}, and we observe convergent behaviors for all $\alpha_0$ when $D\geq15$, so we choose $D=20$ in all examples and only leave $\alpha_0$ to be tuned by users.

\begin{figure}[t]
\vspace{-5px}
\centering
\includegraphics[width=\linewidth]{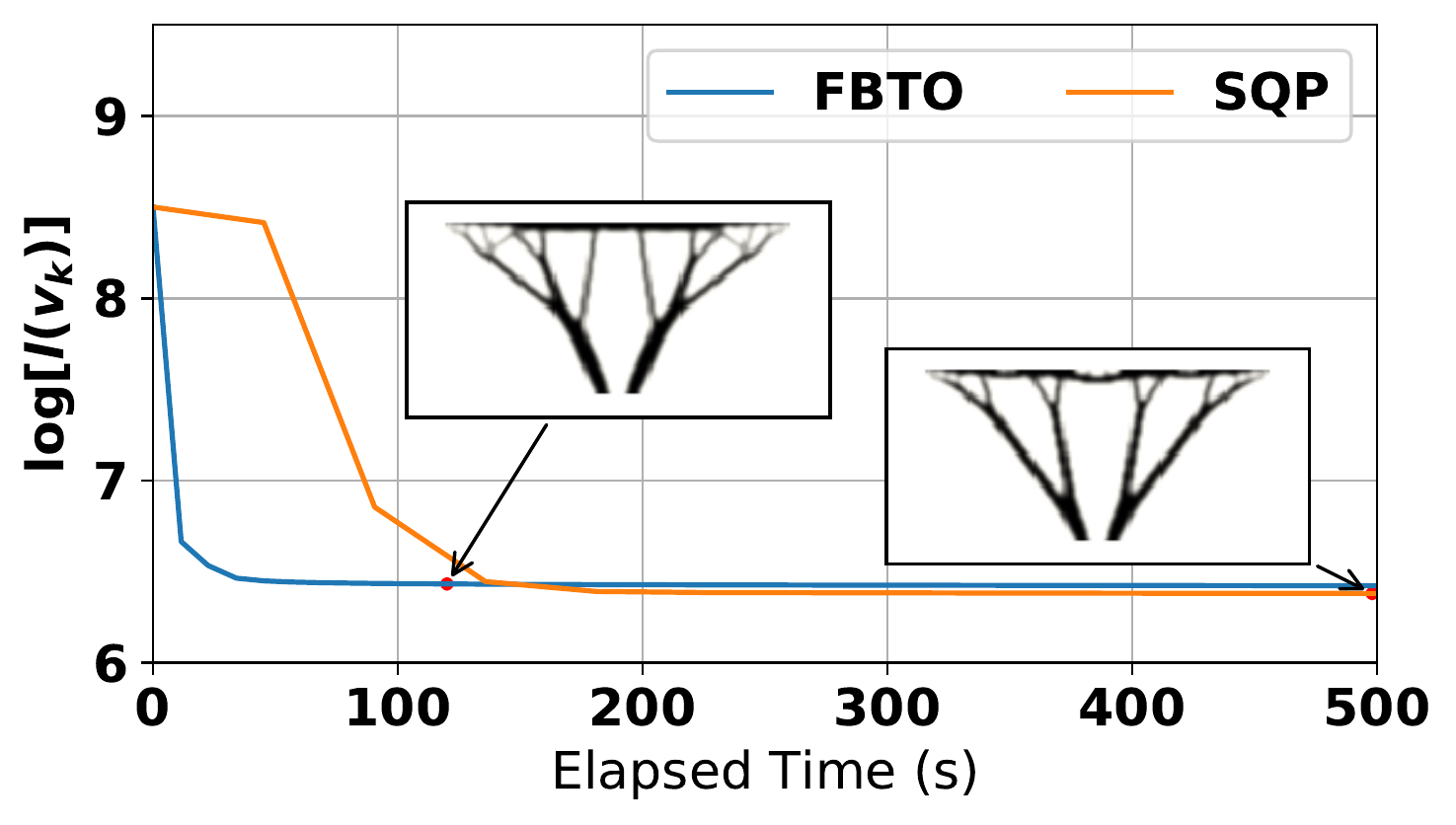}
\caption{\label{fig:compareSQP} \small{We compare SQP and FBTO on the benchmark problem (\prettyref{fig:gallery}c) using a resolution of $64\times128$. While SQP ultimately converges to a better solution, the two solutions are visually similar. Our method finds an approximate solution much more quickly, allowing designers to interactively preview.}}
\vspace{-5px}
\end{figure}
\paragraph{Comparison with SQP} SQP is an off-the-shelf algorithm that pertains global convergence guarantee \cite{rojas2016efficient} and have been used for topology optimization. We conduct comparisons with SQP on the benchmark problem in \prettyref{fig:gallery}c. We use the commercial SQP solver \cite{waltz2004knitro} where we avoid computing the exact Hessian matrix of $l(v)$ via L-BFGS approximation and we avoid factorizing the approximate Hessian matrix using a small number of conjugate gradient iterations. The comparisons are conducted at a resolution of $64\times128$, which is one third that of \prettyref{fig:gallery}c. Using an exact Hessian factorization or a larger resolution in SQP would use up the memory of our desktop machine. As shown in \prettyref{fig:compareSQP}, our method converges to an approximate solution much quicker than SQP, allowing users to preview the results interactively. Such performance of our method is consistent with many other first-order methods, e.g., the Alternating Method Direction of Multiplier (ADMM), which is faster in large-scale problems and converges to a solution of moderate accuracy, while SQP takes longer to converge to a more accurate solution. 

\begin{figure}[t]
\vspace{-5px}
\centering
\includegraphics[width=\linewidth]{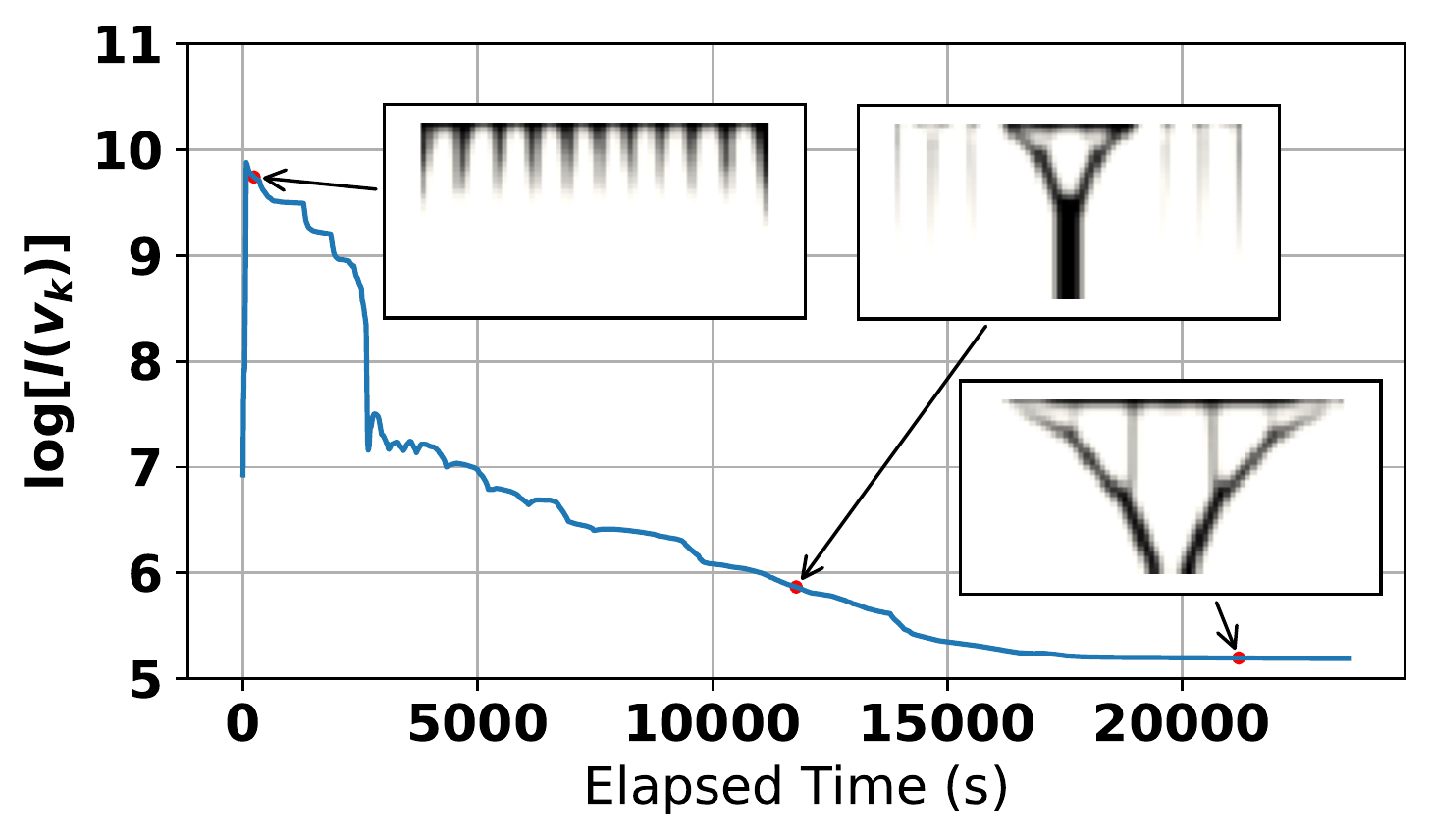}
\caption{\label{fig:compareALM} \small{We run ALM on the benchmark problem (\prettyref{fig:gallery}c) at a resolution of $32\times64$. ALM only exhibits useful solution after $12600$s ($3.5$hr) and converge in $22680$s ($6.3$hr).}}
\vspace{-5px}
\end{figure}
\paragraph{Comparison with ALM} ALM formulates $u,v$ both as decision variables and treat FEA as additional hard constraints:
\begin{align*}
\argmin{v\in X, u}\frac{1}{2}u^TK(v)u\quad\ST\;K(v)u=f.
\end{align*}
By turning the hard constraints into soft penalty functions via the augmented Lagrangian penalty term, ALM stands out as another solver for TO that allows inexact matrix inversions. Although we are unaware of ALM being used for TO problems, it has been proposed to solve general bilevel programs in \cite{mehra2021penalty}. We compare our method with the widely used ALM solver \cite{vanderbei1999loqo}. As the major drawback, ALM needs to square the FEA system matrix $K(v)$, which essentially squares its condition number and slows down the convergence. We have observed this phenomenon in \prettyref{fig:compareALM}, where we run ALM at a resolution of $32\times64$ which is one-sixth that of \prettyref{fig:gallery}c. Even at such a coarse resolution, ALM cannot present meaningful results until after $12600$s ($3.5$hr) that is much longer than either FBTO or SQP. Our \prettyref{alg:POpt} suffers from the same problem of squaring the system matrix. Therefore, we suggest using \prettyref{alg:POpt} only when one has a highly efficient, non-commuting preconditioner such as a geometric multigrid.
\section{\label{sec:conclusion}Conclusion \& Limitation}
We present a novel approach to solve a large subclass of TO problems known as the SIMP model. Unlike classical optimization solvers that perform the exact sensitivity analysis during each iteration, we propose to use inexact sensitivity analysis that is refined over iterations. We show that this method corresponds to the first-order algorithm for a bilevel reformulation of TO problems. Inspired by the recent analytical result \cite{ghadimi2018approximation}, we show that this approach has guaranteed convergence to a first-order critical point of the original problem. Our new approach leads to low-iterative cost and allows interactive result preview for mechanical designers. We finally discuss extensions to use preconditioners and fine-grained GPU parallelism. Experiments on several 2D benchmarks from \cite{valdez2017topology,7332965,kambampati2020large} highlight the computational advantage of our method.

An inherent limitation of our method is the requirement to estimate step size parameters $\alpha_k,\beta_k$. We have shown that $\beta_k$ is fixed and our results are not sensitive to $\alpha_0$, but users still need to give a rough estimation of $\alpha_0$ that can affect practical convergence speed. Further, our first-order solver is based on plain gradient descendent steps, for which several acceleration schemes are available, e.g. Nesterov \cite{beck2009fast} and Anderson \cite{zhang2020globally} accelerations. The analysis of these techniques are left as future works. Finally, when additional constraints are incorporated into the polytopic set $X$ or non-polytopic constraints are considered, our analysis needs to be adjusted and the efficacy of \prettyref{alg:Proj} might be lost.
\printbibliography
\ifsupp
\begin{table}[t]
\caption{\label{table:symbols}Table of Symbols}
\resizebox{.48\textwidth}{!}{
\rowcolors{0}{gray!50}{white}
\begin{tabular}{ll}
\toprule
Variable & Definition \\
\midrule
$x\in\Omega$ & a point in material domain  \\
$u$ & displacement field \\
$P$ & internal potential energy \\
$k_e$ & stiffness tensor field  \\
$v, v_e$ & infill level, infill level of $e$th element  \\
$u_k, v_k, \delta_k$ & solutions at $k$th iteration \\
$\underline{v}, \bar{v}$ & infill level bounds  \\
$\Omega_e$ & sub-domain of $e$th element    \\
$E, N$ & number of elements, nodes   \\
$K, K_e$ & stiffness matrix, stiffness of $e$th element  \\
$f$ & external force field  \\
$P_f$ & total energy    \\
$e_i$ & unit vector at $i$th element    \\
$u_f$ & displacement caused by force $f$    \\
$l$ & induced internal potential energy \\
$\eta$ & power law coefficient   \\
$\rho(\bullet)$ & eigenvalue function \\
$\underline{\rho}, \bar{\rho}$ & eigenvalue bounds  \\
$\ALLONE, \ALLZERO$ & all-one, all-zero vectors    \\
$\mathcal{A}_e$ & activation function of $e$th element  \\
$\mathbb{S}$ & symmetry mapping matrix \\
$v_s$ & infill levels of the left-half material block   \\
$\#$ & number of sub-components \\
$\bar{v}^i$ & total infill level of $i$th component \\
$\mathbb{C}$ & material filter operator    \\
$\alpha_k, \beta_k$ & high-level, low-level step size   \\
$\Delta_k, \Delta_k^v$ & low-level, high-level error metrics    \\
$\Xi_k$ & low-level error metric used to analyze \prettyref{alg:CPOpt}  \\
$U$ & upper bound of $\sqrt{\Delta_1}$ \\
$\Gamma_k, \Theta_k$ & coefficients of low-level error    \\
$\bar{\Gamma}, \bar{\Theta}$ & upper bounds of $\Gamma_k, \Theta_k$  \\
$p, m, n$ & algorithmic constants   \\
$L_u, L_K L_{\Delta K}, L_{\nabla l}$ & L-constants of $u_f, K, \Delta K, \nabla l$ \\
$C, C_v$ & coefficients of reduction rate of $\Delta_k^v$ \\
$\mathcal{T}_X$ & tangent cone of $X$   \\
$d_k, \tau_k$ & feasible direction in $\mathcal{T}_X$ and step size \\
$\underline{l}$ & lower bound of $l$    \\
$U_f$ & uniform upper bound of $\|u_f\|$    \\
$\nabla_X$ & projected gradient into normal cone \\
$\epsilon$ & error tolerance of gradient norm   \\
$\Proj{\bullet}$ & projection operator onto $X$   \\
$\kappa$ & condition number \\
$M$ & preconditioner    \\
$\underline{\rho}_M, \bar{\rho}_M$ & eigenvalue bounds of preconditioner  \\
$c_i$ & coefficients of Krylov vectors   \\
$D$ & dimension of Krylov subspace   \\
$\tilde{\nabla}l$ & approximate gradient    \\
$\tilde{\nabla}_Pl$ & mean-projected approximate gradient    \\
$\lambda_{\ALLONE}, \lambda_{\underline{v}}$ & Lagrangian multiplier for projection problem  \\
\hline
\end{tabular}}
\vspace{-10px}
\end{table}
\begin{figure*}[t]
\setlength{\tabcolsep}{0pt}
\begin{tabu}{ccc}
\includegraphics[width=.32\linewidth]{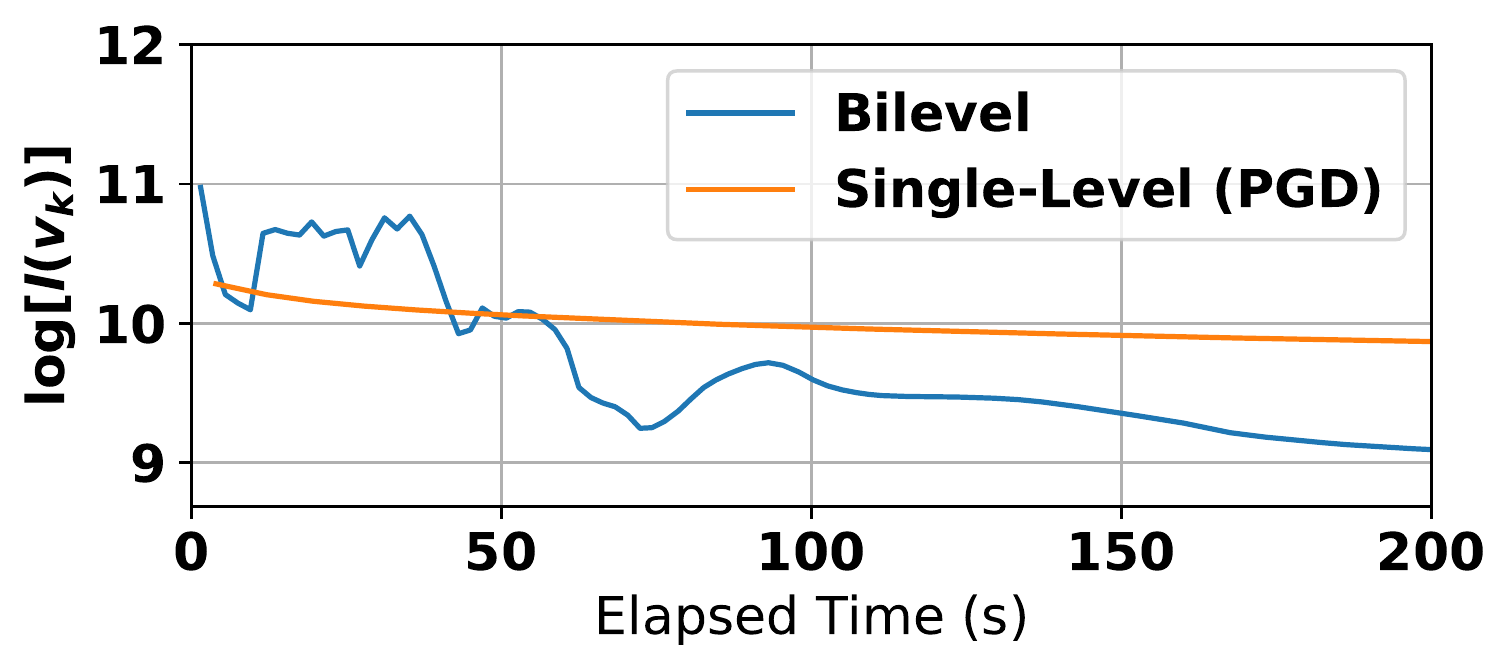}
\put(-120,0){(a)}&
\includegraphics[width=.32\linewidth]{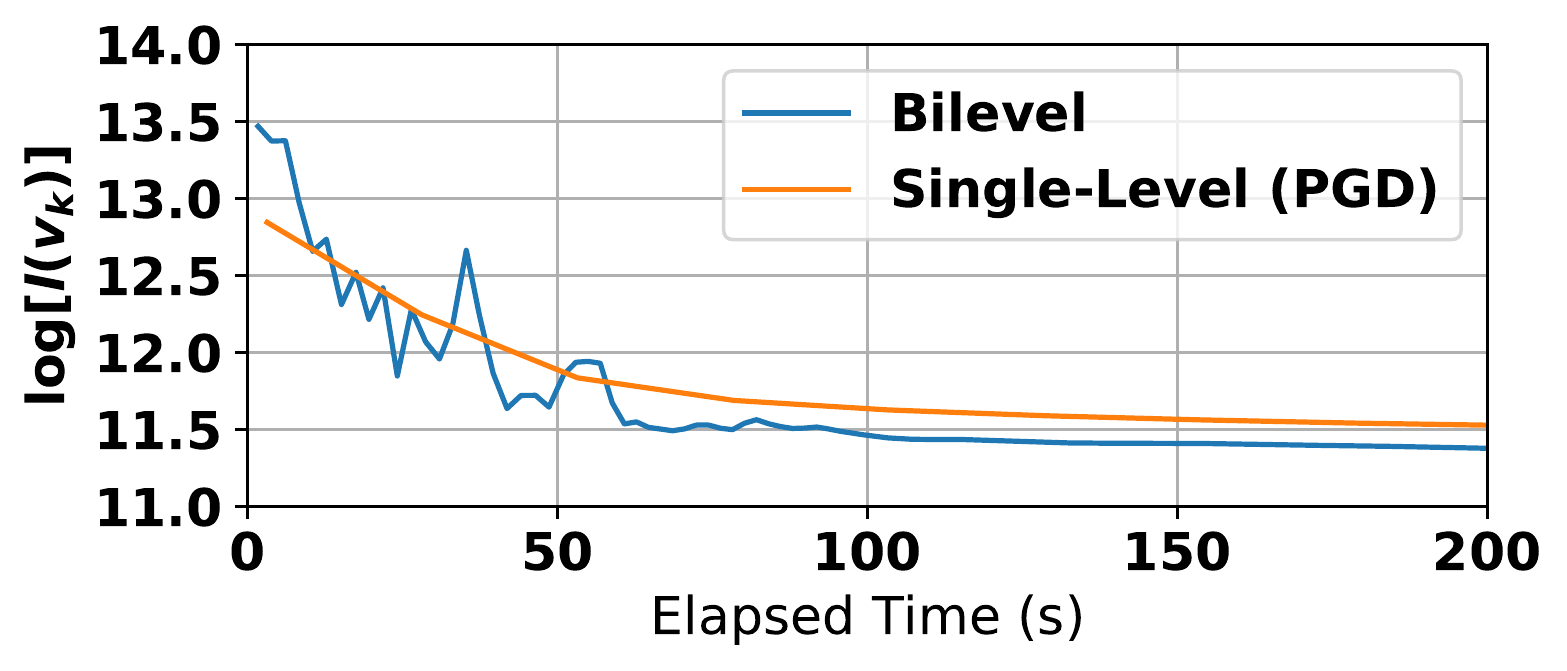}
\put(-120,0){(b)}&
\includegraphics[width=.32\linewidth]{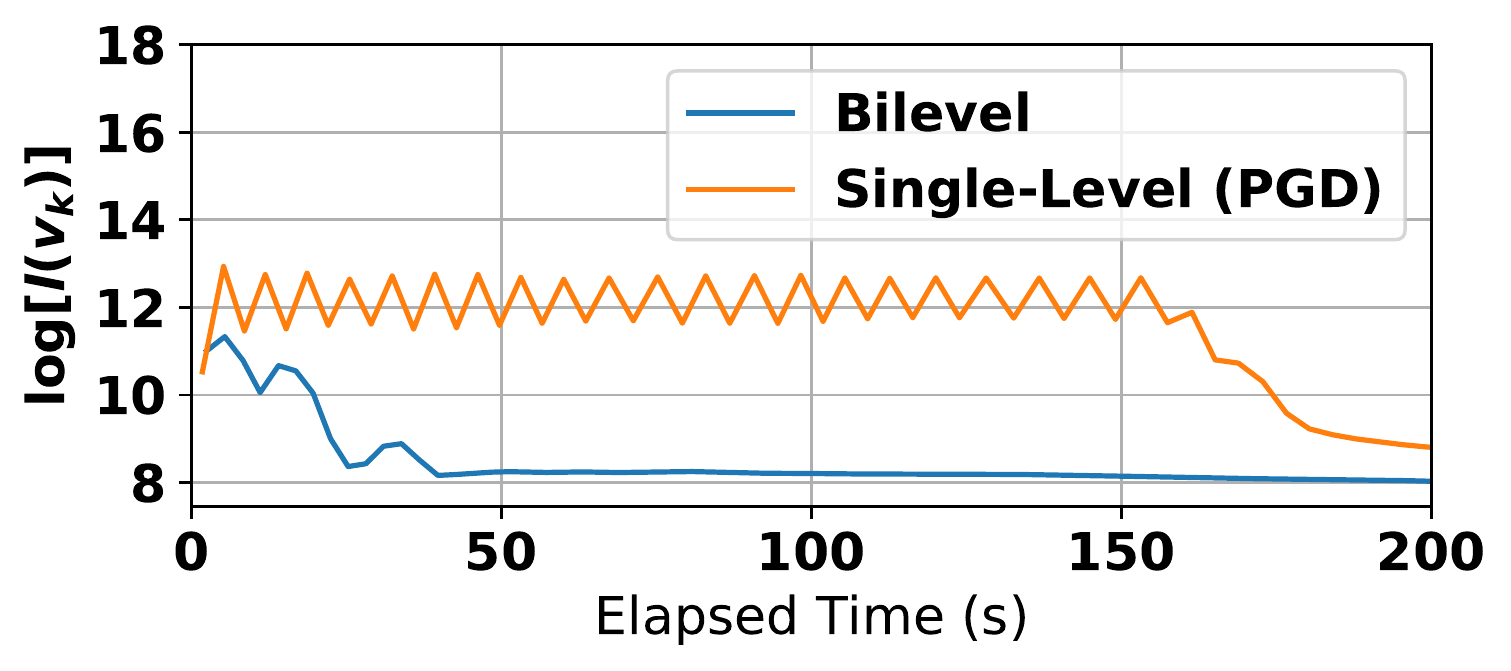}
\put(-120,0){(c)}\\
\includegraphics[width=.32\linewidth]{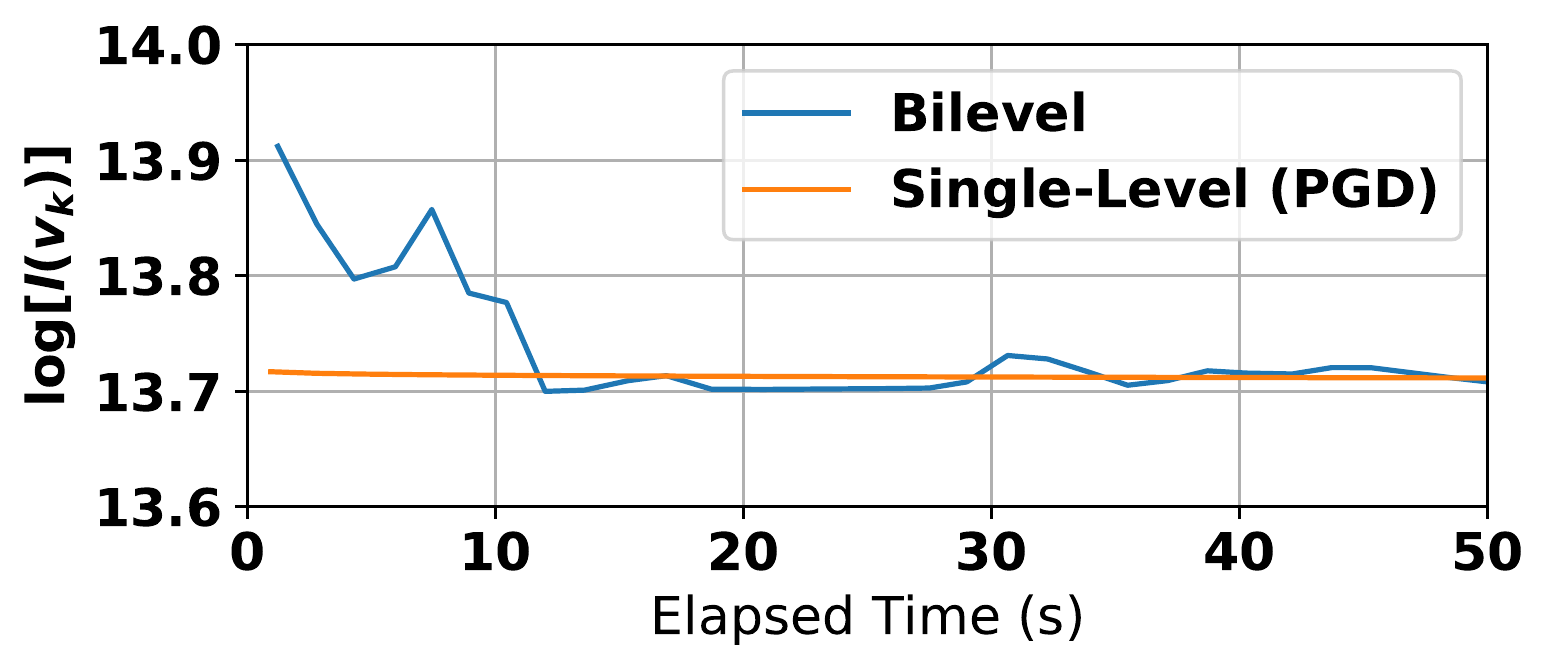}
\put(-120,0){(d)}&
\includegraphics[width=.32\linewidth]{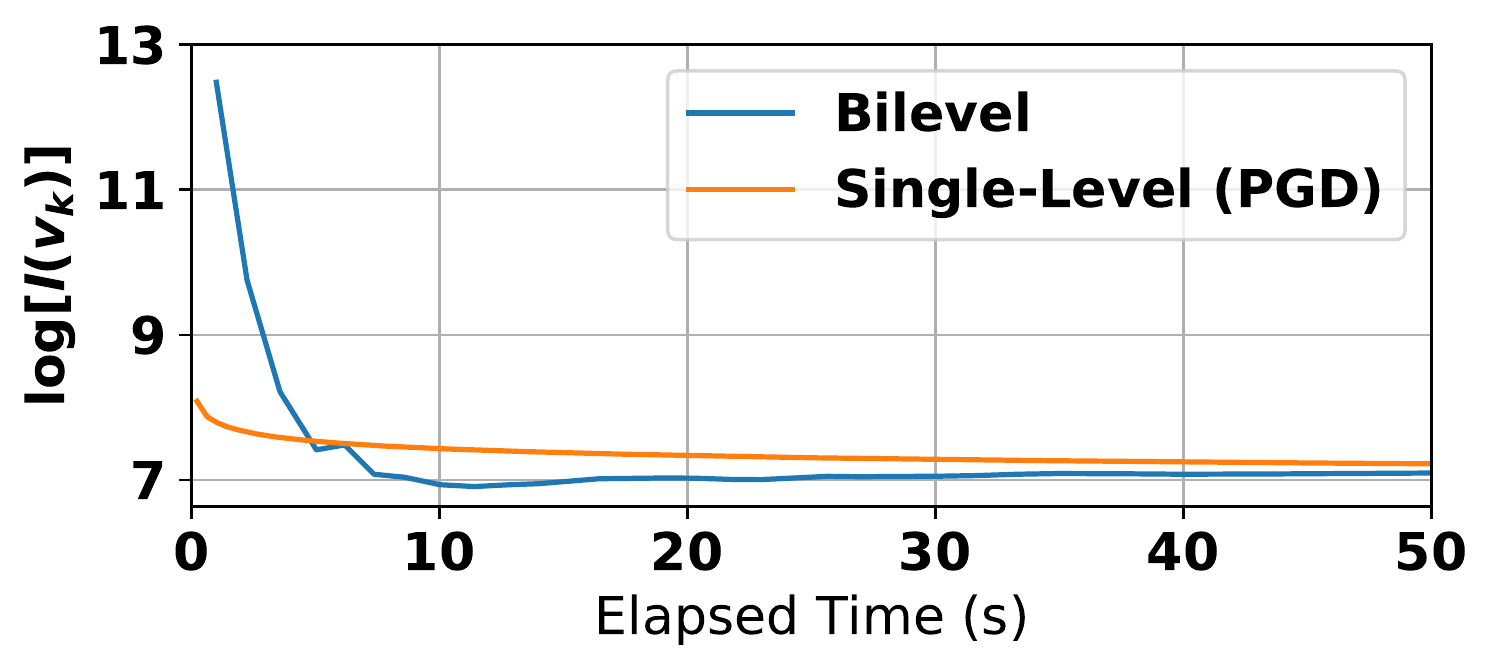}
\put(-120,0){(e)}&
\includegraphics[width=.32\linewidth]{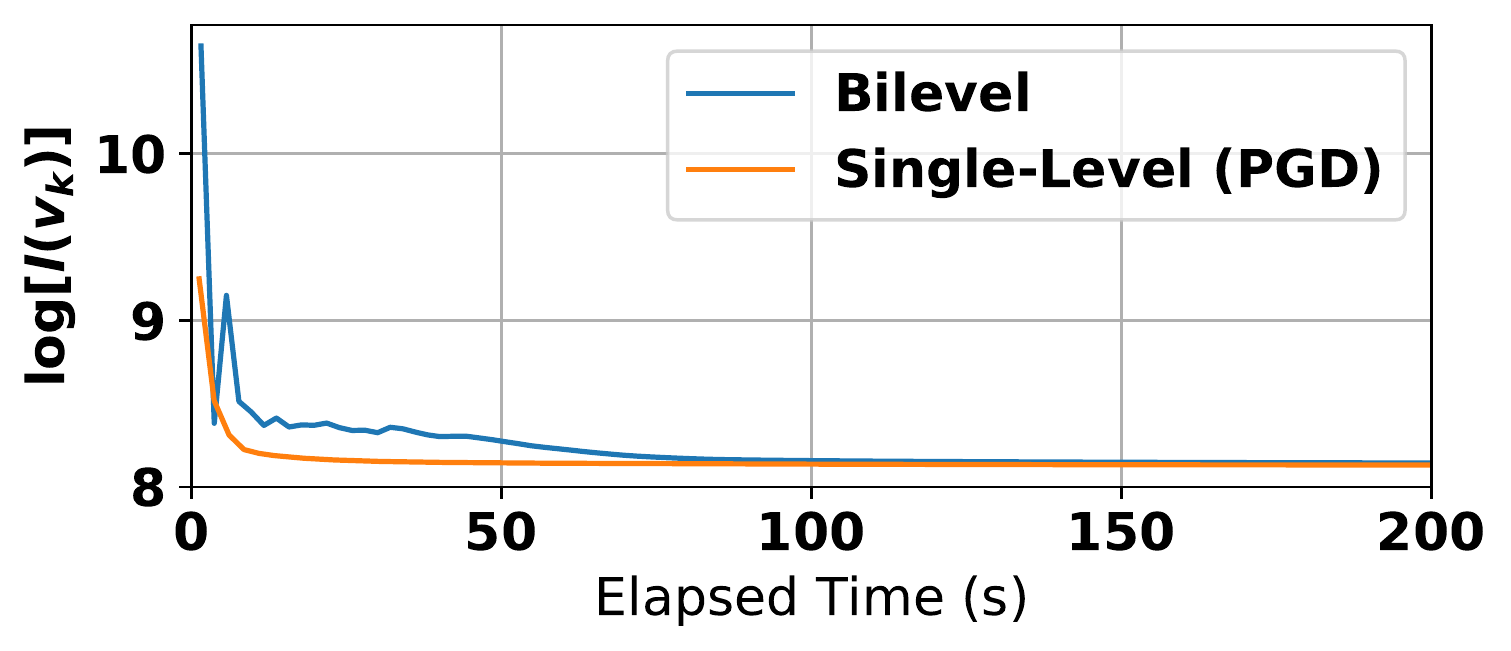}
\put(-120,0){(f)}\\
\includegraphics[width=.32\linewidth]{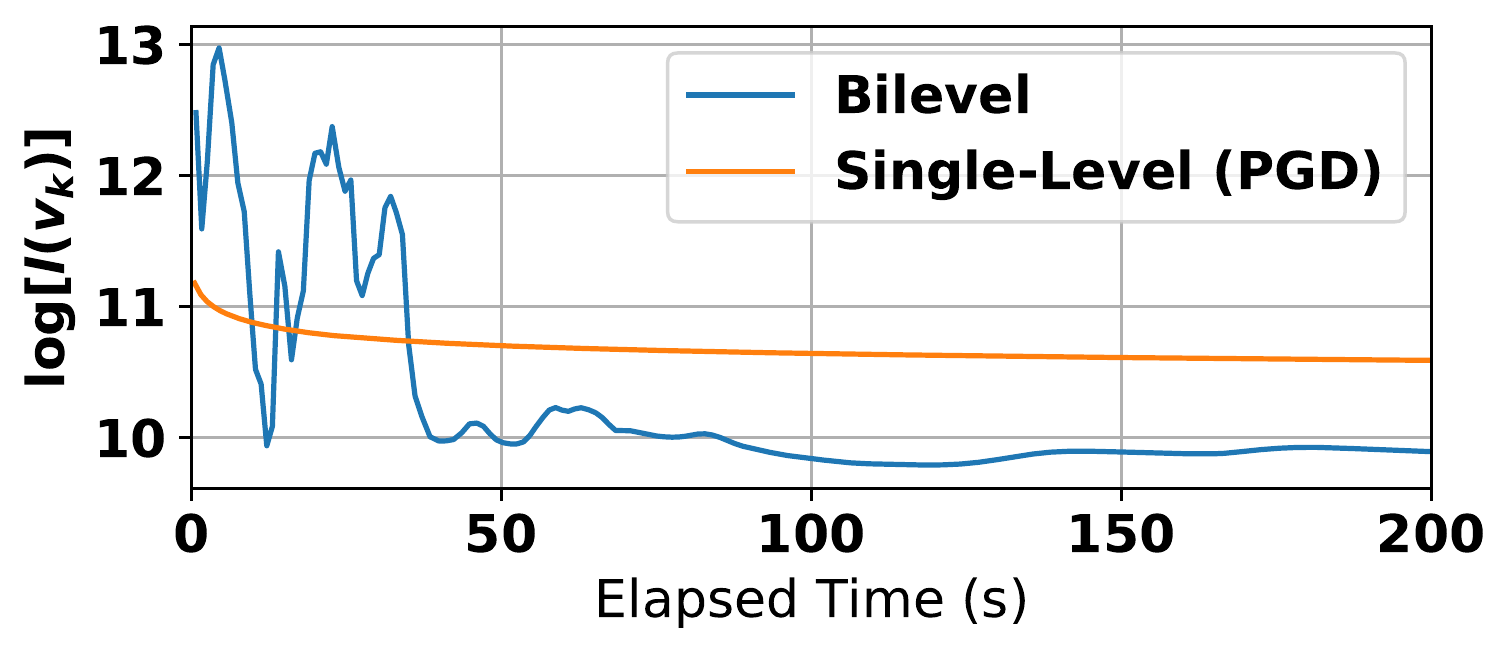}
\put(-120,0){(g)}&
\multicolumn{2}{c}{\multirow{1}{*}[0.7pt]{\parbox{.6\linewidth}{\caption{\label{fig:allBench}We summarize the comparative convergence history of PGD and FBTO for all the benchmarks (\prettyref{fig:gallery}a-g). PGD and FBTO are implemented on GPU with matrix factorization of PGD on CPU. All other parameters are the same.}}}}
\end{tabu}
\vspace{-10px}
\end{figure*}
\section{\label{sec:additionalExpr}Additional Experiments}
We compare the performance of FBTO and PGD for all the benchmarks in \prettyref{fig:gallery} and the results are summarized in \prettyref{fig:allBench}. Note that PGD uses an exact low-level solver so it requires less iterations then FBTO, while FBTO uses more iterations with a much lower iterative cost. For fairness we compare them based on computational time. PGD outperforms or performs similarly to FBTO on two of the problems (\prettyref{fig:allBench}df), where PGD converges rapidly in $5-10$ iterations while FBTO requires thousands of iterations. In all other cases, FBTO converges faster to an approximate solution.
\section{\label{sec:proof}Convergence Analysis of \prettyref{alg:Opt}}
We provide proof of \prettyref{thm:gradientConvergence}. We start from some immediate observations on the low- and high-level objective functions. Then, we prove the rule of error propagation for the low-level optimality, i.e., the low-level different between $u_k$ and its optimal solution $u_f(v_k)$. Next, we choose parameters to let the difference diminish as the number of iterations increase. Finally, we focus on the high-level objective function and show that the difference between $v_k$ and a local optima of $l(v)$ would also diminish. Our analysis resembles the recent analysis on two-timescale bilevel optimization, but we use problem-specific treatment to handle our novel gradient estimation for the high-level problem without matrix inversion.

\subsection{Low-Level Error Propagation}
If \prettyref{ass:StiffnessParameter} holds, then the low-level objective function is $\underline{\rho}$-strongly convex and Lipschitz-continuous in $u$ for any $v$ with $\bar{\rho}$ being the L-constant, so that:
\small
\begin{align*}
&P_f(u_{k+1},v_k)\leq 
P_f(u_k,v_k)+\left<u_{k+1}-u_k,\nabla_u P_f(u_k,v_k)\right>+\\
&\frac{\bar{\rho}}{2}\|u_{k+1}-u_k\|^2
\leq P_f(u_k,v_k)+\left[\frac{\bar{\rho}}{2}-\frac{1}{\beta_k}\right]\|u_{k+1}-u_k\|^2.\numberthis\label{eq:lowLevelUpdate}
\end{align*}
\normalsize
Similarly, we can immediately estimate the approximation error of the low-level problem due to an update on $u$:
\begin{lemma}
\label{lem:lowLevelError}
Under \prettyref{ass:StiffnessParameter}, the following relationship holds for all $k\geq1$:
\small
\begin{align*}
\Delta_{k+1}\leq\frac{p+1}{p}(1-2\beta_k\underline{\rho}+\beta_k^2\bar{\rho}^2)\Delta_k+
L_u^2L_{\nabla K}^2\alpha_k^2\frac{p+1}{4}\|u_k\|^4.
\end{align*}
\normalsize
\end{lemma}
\begin{proof}
The following result holds by the triangle inequality and the bounded spectrum of $K(v)$ (\prettyref{eq:KVBound}):
\begin{align*}
&\|u_{k+1}-u_f(v_k)\|^2\\
=&\|u_{k+1}-u_k\|^2+\|u_k-u_f(v_k)\|^2+\\
&2\left<u_{k+1}-u_k,u_k-u_f(v_k)\right>\\
=&\|u_{k+1}-u_k\|^2+\|u_k-u_f(v_k)\|^2-\\
&2\beta_k\left<K(v_k)u_k-f,u_k-u_f(v_k)\right>\\
\leq&\|u_{k+1}-u_k\|^2+(1-2\beta_k\underline{\rho})\|u_k-u_f(v_k)\|^2\\
\leq&\beta_k^2\bar{\rho}^2\|u_k-u_f(v_k)\|^2+(1-2\beta_k\underline{\rho})\|u_k-u_f(v_k)\|^2.\numberthis\label{eq:lowLevelTrackingUpdate}
\end{align*}
The optimal solution to the low-level problem is $u_f(v)$, which is a smooth function defined on a compact domain, so any derivatives of $u_f(v)$ is bounded L-continuous and we have the following estimate of the change to $u_f$ due to an update on $v$: 
\begin{align*}
&\|u_{k+1}-u_f(v_k)\|\|u_f(v_k)-u_f(v_{k+1})\|\\
\leq&\frac{1}{2p}\|u_{k+1}-u_f(v_k)\|^2+\frac{p}{2}\|u_f(v_k)-u_f(v_{k+1})\|^2\\
\leq&\frac{1}{2p}\|u_{k+1}-u_f(v_k)\|^2+\frac{pL_u^2}{2}\|v_k-v_{k+1}\|^2\\
\leq&\frac{1}{2p}\|u_{k+1}-u_f(v_k)\|^2+\frac{pL_u^2\alpha_k^2}{8}\|u_k^T\FPP{K}{v_k}u_k\|^2,\numberthis\label{eq:lowLevelYoung}
\end{align*}
where we have used Young's inequality and the contractive property of the $\Proj{\bullet}$ operator. Here we denote $L_u$ as the L-constant of $u_f$. By triangular inequality, we further have:
\small
\begin{align*}
&\|u_{k+1}-u_f(v_{k+1})\|^2\\
\leq&\|u_{k+1}-u_f(v_k)\|^2+\|u_f(v_k)-u_f(v_{k+1})\|^2+\\
&2\|u_{k+1}-u_f(v_k)\|\|u_f(v_k)-u_f(v_{k+1})\|\\
\leq&\frac{p+1}{p}\|u_{k+1}-u_f(v_k)\|^2+\|u_f(v_k)-u_f(v_{k+1})\|^2+\\
&\frac{pL_u^2\alpha_k^2}{4}\|u_k^T\FPP{K}{v_k}u_k\|^2\\
\leq&\frac{p+1}{p}\|u_{k+1}-u_f(v_k)\|^2+L_u^2\|v_k-v_{k+1}\|^2+\\
&\frac{pL_u^2\alpha_k^2}{4}\|u_k^T\FPP{K}{v_k}u_k\|^2\\
\leq&\frac{p+1}{p}\|u_{k+1}-u_f(v_k)\|^2+
L_u^2\alpha_k^2\frac{p+1}{4}\|u_k^T\FPP{K}{v_k}u_k\|^2\\
\leq&\frac{p+1}{p}(1-2\beta_k\underline{\rho}+\beta_k^2\bar{\rho}^2)\|u_k-u_f(v_k)\|^2+\\
&L_u^2\alpha_k^2\frac{p+1}{4}\|u_k^T\FPP{K}{v_k}u_k\|^2\\
\leq&\frac{p+1}{p}(1-2\beta_k\underline{\rho}+\beta_k^2\bar{\rho}^2)\|u_k-u_f(v_k)\|^2+\\
&L_u^2L_{\nabla K}^2\alpha_k^2\frac{p+1}{4}\|u_k\|^4,\numberthis\label{eq:lowLevelError}
\end{align*}
\normalsize
where we have used \prettyref{eq:lowLevelTrackingUpdate} and \prettyref{eq:lowLevelYoung}.
\end{proof}
The result of \prettyref{lem:lowLevelError} is a recurrent relationship on the low-level optimality error $\Delta_k$, which will be used to prove low-level convergence via recursive expansion.

\subsection{Low-Level Convergence}
We use the following shorthand notation for the result in \prettyref{lem:lowLevelError}:
\small
\begin{align*}
&\Delta_{k+1}\leq\Theta_k\Delta_k+\Gamma_k\|u_k\|^4   \\
&\Theta_k\triangleq\frac{p+1}{p}(1-2\beta_k\underline{\rho}+\beta_k^2\bar{\rho}^2)\quad
\Gamma_k\triangleq L_u^2L_{\nabla K}^2\alpha_k^2\frac{p+1}{4}.\numberthis\label{eq:lowLevelErrorShort}
\end{align*}
\normalsize
By taking \prettyref{ass:BetaChoice} and direct calculation, we can ensure that $\Theta_k\leq\bar{\Theta}<1$ (i.e., the first term is contractive). To bound the growth of the second term above, we show by induction that both $\Delta_k$ and $u_k$ can be uniformly bounded for all $k$ via a sufficiently small, constant $\Gamma_k\leq\bar{\Gamma}$.
\begin{lemma}
\label{lem:UpperBound}
Taking \prettyref{ass:StiffnessParameter}, \ref{ass:BetaChoice}, \ref{ass:GammaChoice}, we have $\Delta_k\leq U^2,\|u_k\|\leq U_f+U$ for all $k\geq1$, where $U_f$ is the uniform upper bound for $\|u_f(v)\|$.
\end{lemma}
\begin{proof}
First, since $v\in X$ and $X$ is compact, we have $\|u_f(v)\|\leq U_f<\infty$. Next, we prove $\Delta_k\leq U^2$ by induction. We already have $\Delta_1\leq U^2$. Now suppose $\Delta_k\leq U^2$, then \prettyref{eq:lowLevelErrorShort} and our assumption on $\bar{\Gamma}$ immediately leads to:
\begin{align*}
\Delta_{k+1}\leq&\bar{\Theta}U^2+\bar{\Gamma}(\|u_k-u_f(v_k)\|+\|u_f(v_k)\|)^4\\
\leq&\bar{\Theta}U^2+\bar{\Gamma}(U+U_f)^4\leq U^2.
\end{align*}
Finally, we have: $\|u_k\|\leq\|u_k-u_f(v_k)\|+\|u_f(v_k)\|\leq U+U_f$ and our lemma follows.
\end{proof}
The shrinking coefficient $\bar{\Theta}$ and the uniform boundedness of $\Delta_k,u_k$ allows us to establish low-level convergence with the appropriate choice of $\alpha_k=\mathcal{O}(k^{-m})$ with $m\geq1$.
\begin{theorem}
\label{thm:lowLevelConvergence}
Taking \prettyref{ass:StiffnessParameter}, \ref{ass:BetaChoice}, \ref{ass:GammaChoice}, \ref{ass:AlphaChoice}, we can upper bound $\Delta_{k+1}$ as:
\small
\begin{align*}
\Delta_{k+1}\leq\bar{\Theta}^{k-1}\Delta_1+(U+U_f)^4\Gamma_1\sum_{i=0}^{k-1}\frac{\bar{\Theta}^i}{(k-i)^{2m }}=\mathcal{O}(k^{1-2m}).
\end{align*}
\normalsize
\end{theorem}
\begin{proof}
Recursively expand on \prettyref{eq:lowLevelErrorShort} and we have:
\small
\begin{align*}
&\Delta_{k+1}\leq\bar{\Theta}\Delta_k+\Gamma_k\|u_k\|^4\\
\leq&\bar{\Theta}(\bar{\Theta}\Delta_{k-1}+\Gamma_{k-1}\|u_{k-1}\|^4)+\Gamma_k\|u_k\|^4\\
=&\bar{\Theta}^2\Delta_{k-1}+\bar{\Theta}\Gamma_{k-1}\|u_{k-1}\|^4+\Gamma_k\|u_k\|^4\\
\leq&\bar{\Theta}^2(\bar{\Theta}\Delta_{k-2}+\Gamma_{k-2}\|u_{k-2}\|^4)+\bar{\Theta}\Gamma_{k-1}\|u_{k-1}\|^4+\Gamma_k\|u_k\|^4\\
&\cdots\\
\leq&\bar{\Theta}^k\Delta_1+(U+U_f)^4\sum_{i=0}^{k-1}\bar{\Theta}^i\Gamma_{k-i}\\
\leq&\bar{\Theta}^k\Delta_1+(U+U_f)^4\Gamma_1\sum_{i=0}^{k-1}\frac{\bar{\Theta}^i}{(k-i)^{2m}}\\
\leq&\bar{\Theta}^k\Delta_1+(U+U_f)^4\Gamma_1\left[\sum_{i=0}^{\lceil\frac{k-1}{2}\rceil}\frac{\bar{\Theta}^i}{(k-i)^{2m}}+
\sum_{i=\lceil\frac{k+1}{2}\rceil}^{k-1}\frac{\bar{\Theta}^i}{(k-i)^{2m}}\right]\\
\leq&\bar{\Theta}^k\Delta_1+(U+U_f)^4\Gamma_1\left[\frac{\lceil\frac{k+1}{2}\rceil}{(k-\lceil\frac{k-1}{2}\rceil)^{2m}}+\frac{\bar{\Theta}^{\lceil\frac{k}{2}\rceil}}{1-\bar{\Theta}}\right]
=\mathcal{O}(k^{1-2m}),
\end{align*}
\normalsize
where we have used our choice of $\alpha_k$ and \prettyref{lem:UpperBound}.
\end{proof}
\prettyref{thm:lowLevelConvergence} is pivotal by allowing us to choose $m$ and tune the convergence speed of the low-level problem, which is used to establish the convergence of high-level problem.
\subsection{High-Level Error Propagation}
The high-level error propagation is similar to the low-level analysis, which is due to the L-continuity of any derivatives of $l(v)$ in the compact domain $X$. The following result reveals the rule of error propagation over a single iteration:
\begin{lemma}
\label{lem:highLevelError}
Under \prettyref{ass:StiffnessParameter}, \ref{ass:BetaChoice}, \ref{ass:GammaChoice}, \ref{ass:AlphaChoice}, the following high-level error propagation rule holds for all $k\geq1,q>0$:
\begin{align*}
&l(v_{k+1})-l(v_k)\leq-\frac{1}{\alpha_k}\|v_{k+1}-v_k\|^2+\\
&\frac{L_{\nabla K}^2\alpha_k^2(L_{\nabla l}q+1)}{8q}(U+U_f)^4+\frac{L_{\nabla K}^2q}{8}\Delta_k(U+2U_f)^2.
\end{align*}
\end{lemma}
\begin{proof}
By the smoothness of $l(v)$, the compactness of $X$, and the obtuse angle criterion, we have:
\small
\begin{align*}
&l(v_{k+1})-l(v_k)\leq\left<v_{k+1}-v_k,\nabla l(v_k)\right>+\frac{L_{\nabla l}}{2}\|v_{k+1}-v_k\|^2\\
\leq&\left<v_{k+1}-v_k,-\frac{1}{2}u_f^T(v_k)\FPP{K}{v_k}u_f(v_k)\right>+\frac{L_{\nabla l}L_{\nabla K}^2\alpha_k^2}{8}\|u_k\|^4\\
=&\left<v_{k+1}-v_k,-\frac{1}{2}u_k^T\FPP{K}{v_k}u_k\right>+\frac{L_{\nabla l}L_{\nabla K}^2\alpha_k^2}{8}\|u_k\|^4+\\
&\left<v_{k+1}-v_k,\frac{1}{2}u_k^T\FPP{K}{v_k}u_k-\frac{1}{2}u_f^T(v_k)\FPP{K}{v_k}u_f(v_k)\right>\\
\leq&-\frac{1}{\alpha_k}\|v_{k+1}-v_k\|^2+\frac{L_{\nabla l}L_{\nabla K}^2\alpha_k^2}{8}\|u_k\|^4+\\
&\left<v_{k+1}-v_k,\frac{1}{2}u_k^T\FPP{K}{v_k}u_k-\frac{1}{2}u_f^T(v_k)\FPP{K}{v_k}u_f(v_k)\right>,\numberthis\label{eq:highLevelConvexity}
\end{align*}
\normalsize
where $L_{\nabla l}$ is the L-constant of $\nabla l$ on $X$. For the last term above, we can bound it by applying the Young's inequality:
\small
\begin{align*}
&\left<v_{k+1}-v_k,\frac{1}{2}u_k^T\FPP{K}{v_k}u_k-\frac{1}{2}u_f^T(v_k)\FPP{K}{v_k}u_f(v_k)\right>\\
\leq&\frac{1}{2q}\|v_{k+1}-v_k\|^2+\frac{q}{2}\|\frac{1}{2}u_k^T\FPP{K}{v_k}u_k-\frac{1}{2}u_f^T(v_k)\FPP{K}{v_k}u_f(v_k)\|^2\\
\leq&\frac{L_{\nabla K}^2\alpha_k^2}{8q}\|u_k\|^4+
\frac{q}{2}\|\frac{1}{2}u_k^T\FPP{K}{v_k}u_k-\frac{1}{2}u_f^T(v_k)\FPP{K}{v_k}u_f(v_k)\|^2\\
=&\frac{L_{\nabla K}^2\alpha_k^2}{8q}\|u_k\|^4+
\frac{q}{2}\|\frac{1}{2}(u_k-u_f(v_k))^T\FPP{K}{v_k}(u_k+u_f(v_k))\|^2\\
\leq&\frac{L_{\nabla K}^2\alpha_k^2}{8q}\|u_k\|^4+\frac{L_{\nabla K}^2q}{8}\Delta_k\|u_k+u_f(v_k)\|^2.\numberthis\label{eq:highLevelYoung}
\end{align*}
\normalsize
The lemma follows by combining \prettyref{eq:highLevelConvexity}, \prettyref{eq:highLevelYoung}, and \prettyref{lem:UpperBound}.
\end{proof}

\subsection{High-Level Convergence}
We first show that $\Delta_k^v$ is diminishing via the follow lemma:
\begin{lemma}
\label{lem:highLevelConvergence}
Under \prettyref{ass:StiffnessParameter}, \ref{ass:BetaChoice}, \ref{ass:GammaChoice}, \ref{ass:AlphaChoice}, we have the following bound on the accumulated high-level error $\Delta_k^v$:
\begin{align*}
&\frac{1}{2}\sum_{i=1}^k\alpha_i\Delta_i^v\leq l(v_1)-\underline{l}+\\
&\frac{L_{\nabla K}^2}{8}(U+U_f)^4\sum_{i=1}^k\alpha_i^{2-n}(L_{\nabla l}\alpha_i^n+1)+\\
&\frac{L_{\nabla K}^2}{8}(U+2U_f)^2\sum_{i=1}^k\alpha_i^n\Delta_i+
\frac{L_{\nabla K}^2}{4}(U+2U_f)^2\sum_{i=1}^k\alpha_i\Delta_i,
\end{align*}
where $\underline{l}$ is the lower bound of $l$ on $X$.
\end{lemma}
\begin{proof}
By choosing $q=\alpha_k^n$ for some constant $n$ and summing up the recursive rule of \prettyref{lem:highLevelError}, we have the following result:
\begin{align*}
&l(v_{k+1})-l(v_1)\leq-\sum_{i=1}^k\frac{1}{\alpha_i}\|v_{i+1}-v_i\|^2+\\
&\frac{L_{\nabla K}^2}{8}(U+U_f)^4\sum_{i=1}^k\alpha_i^{2-n}(L_{\nabla l}\alpha_i^n+1)+\\
&\frac{L_{\nabla K}^2}{8}(U+2U_f)^2\sum_{i=1}^k\alpha_i^n\Delta_i.\numberthis\label{eq:highLevelAccumulation}
\end{align*}
\prettyref{eq:highLevelAccumulation} can immediately lead to the $\mathcal{O}(k^{-1})$ convergence of $\|v_{i+1}-v_i\|^2/\alpha_i$. However, the update from $v_i$ to $v_{i+1}$ is using the approximate gradient. To show the convergence of $l(v)$, we need to consider an update using the exact gradient. This can be achieved by combining \prettyref{eq:highLevelAccumulation} and the gradient error estimation in \prettyref{eq:highLevelYoung}:
\small
\begin{align*}
&\Delta_k^v\leq\frac{1}{\alpha_k^2}\|v_{k+1}-v_k\|^2+\\
&\frac{1}{\alpha_k^2}\|\frac{\alpha_k}{2}u_f(v_k)^T\FPP{K}{v_k}u_f(v_k)-\frac{\alpha_k}{2}u_k^T\FPP{K}{v_k}u_k\|^2+\\
&\frac{2}{\alpha_k^2}\|v_{k+1}-v_k\|\|\frac{\alpha_k}{2}u_f(v_k)^T\FPP{K}{v_k}u_f(v_k)-\frac{\alpha_k}{2}u_k^T\FPP{K}{v_k}u_k\|\\
\leq&\frac{2}{\alpha_k^2}\|v_{k+1}-v_k\|^2+
\frac{2}{\alpha_k^2}\|\frac{\alpha_k}{2}u_f(v_k)^T\FPP{K}{v_k}u_f(v_k)-\frac{\alpha_k}{2}u_k^T\FPP{K}{v_k}u_k\|^2\\
\leq&\frac{2}{\alpha_k^2}\|v_{k+1}-v_k\|^2+\frac{L_{\nabla K}^2}{2}\Delta_k(U+2U_f)^2.\numberthis\label{eq:highLevelTrueError}
\end{align*}
\normalsize
We can prove our lemma by combining \prettyref{lem:highLevelError}, \prettyref{eq:highLevelAccumulation}, \prettyref{eq:highLevelTrueError}.
\end{proof}

The last three terms on the righthand side of \prettyref{lem:highLevelConvergence} are power series, the summand of which scales at the speed of $\mathcal{O}(k^{-2m+mn}), \mathcal{O}(k^{1-2m-mn}), \mathcal{O}(k^{1-3m})$, respectively. Therefore, for the three summation to be upper bounded for arbitrary $k$, we need the first condition in \prettyref{ass:MNChoice}. The following corollary is immediate:
\begin{corollary}
\label{cor:projectionError}
Under \prettyref{ass:StiffnessParameter}, \ref{ass:BetaChoice}, \ref{ass:GammaChoice}, \ref{ass:AlphaChoice}, \ref{ass:MNChoice}, we have $\min_{i=1,\cdots,k}\alpha_i\Delta_i^v\leq Ck^{-1}$ and there are infinitely many $k$ such that $\Delta_k^v\leq C_vk^{m-1}$ for some constants $C,C_v$ independent of $k$.
\end{corollary}
\begin{proof}
By \prettyref{lem:highLevelConvergence} and \prettyref{ass:MNChoice}, we have $\sum_{i=1}^k\alpha_i\Delta_i^v\leq C$ for some constant $C$. Now suppose only finitely many $k$ satisfies $\Delta_k^v\leq C_vk^{m-1}$, then after sufficiently large $k\geq K_0$, we have:
\begin{align*}
\sum_{i=K_0}^k\alpha_i\Delta_i^v\geq
\sqrt{\frac{\bar{\Gamma}}{L_u^2L_{\nabla K}^2}\frac{4}{p+1}}C_v\sum_{i=K_0}^k\frac{1}{i},
\end{align*}
which is divergent, leading to a contradiction.
\end{proof}
Our remaining argument is similar to the standard convergence proof of PGD \cite{calamai1987projected}, with minor modification to account for our approximate gradient:
\begin{proof}[Proof of \prettyref{thm:gradientConvergence}]
We denote by $\tilde{v}_{k+1}$:
\begin{align*}
\tilde{v}_{k+1}\triangleq\Proj{v_k-\alpha_k\nabla l(v_k)}.
\end{align*}
Note that $\tilde{v}_{k+1}$ is derived from $v_k$ using the true gradient, while $v_{k+1}$ is derived using our approximate gradient. Due to the polytopic shape of $X$ in \prettyref{ass:StiffnessParameter}, for $v_k\in X$, we can always choose a feasible direction $d_k$ from $\mathcal{T}_X(v_k)$ with $\|d_k\|\leq1$ such that:
\begin{align*}
\|\nablaX l(v_k)\|\leq-\left<\nabla l(v_k),d_k\right>+\frac{\epsilon}{4}.
\end{align*}
We further have the follow inequality holds for any $z\in X$ due to the obtuse angle criterion and the convexity of $X$:
\begin{align*}
&\left<\alpha_k\nabla l(v_k),\tilde{v}_{k+1}-z\right>\\
\leq&\left<\alpha_k\nabla l(v_k),\tilde{v}_{k+1}-z\right>+\left<v_k-\alpha_k\nabla l(v_k)-\tilde{v}_{k+1},\tilde{v}_{k+1}-z\right>\\
=&\left<v_k-\tilde{v}_{k+1},\tilde{v}_{k+1}-z\right>\leq\|v_k-\tilde{v}_{k+1}\|\|\tilde{v}_{k+1}-z\|.
\end{align*}
Applying \prettyref{cor:projectionError} and we can choose sufficiently large $k$ such that:
\begin{align*}
&-\left<\nabla l(\tilde{v}_{k+1}),\frac{\tau_kd_k}{\tau_k\|d_k\|}\right>\leq\frac{1}{\alpha_k}\|v_k-\tilde{v}_{k+1}\|\\
\leq&\frac{1}{\alpha_k}\|v_k-v_{k+1}\|+\frac{1}{\alpha_k}\|v_{k+1}-\tilde{v}_{k+1}\|\\
=&\sqrt{\Delta_k^v}+
\frac{1}{\alpha_k}\|v_{k+1}-\tilde{v}_{k+1}\|\\
\leq&\frac{\epsilon}{4}+\sqrt{\frac{L_{\nabla K}^2}{4}\Delta_k\|u_k+u_f(v_k)\|^2}
=\frac{\epsilon}{4}+\mathcal{O}(k^{(1-2m)/2}).
\end{align*}
Here the first inequality holds by choosing sufficiently small $k$-dependent $\tau_k$ such that $z=\tilde{v}_{k+1}+\tau_kd_k\in X$ and using the fact that $\|d_k\|\leq 1$. The third inequality holds by choosing sufficently large $k$ and \prettyref{cor:projectionError}. The last equality holds by the contractive property of projection operator, \prettyref{thm:lowLevelConvergence}, and \prettyref{lem:highLevelError}.
\end{proof}

\subsection{\label{sec:PreconditionedProof}Convergence Analysis of \prettyref{alg:POpt}}
Informally, we establish convergence for the preconditioned \prettyref{alg:POpt} by modifying \prettyref{eq:lowLevelTrackingUpdate} as follows:
\begin{align*}
&\|u_{k+1}-u_f(v_k)\|^2\\
=&\|u_{k+1}-u_k\|^2+\|u_k-u_f(v_k)\|^2+\\
&2\left<u_{k+1}-u_k,u_k-u_f(v_k)\right>\\
=&\|u_{k+1}-u_k\|^2+\|u_k-u_f(v_k)\|^2-\\
&2\beta_k\left<K(v_k)M^{-2}(v_k)(K(v_k)u_k-f),u_k-u_f(v_k)\right>\\
\leq&\|u_{k+1}-u_k\|^2+(1-2\beta_k\frac{\underline{\rho}^2}{\bar{\rho}_M^2})\|u_k-u_f(v_k)\|^2\\
\leq&\beta_k^2\frac{\bar{\rho}^4}{\underline{\rho}_M^4}\|u_k-u_f(v_k)\|^2+(1-2\beta_k\frac{\underline{\rho}^2}{\bar{\rho}_M^2})\|u_k-u_f(v_k)\|^2.
\end{align*}
\prettyref{ass:BetaChoice} must also be modified to account for the spectrum $M(v)$. All the other steps are identical to those of \prettyref{thm:gradientConvergence}.
\section{\label{sec:proofPre} Convergence Analysis of \prettyref{alg:CPOpt}}
The main difference in the analysis of \prettyref{alg:CPOpt} lies in the use of a different low-level error metric defined as $\Xi_k\triangleq\|K(v_k)u_k-f\|^2$. Unlike $\Delta_k$ which requires exact matrix inversion, $\Xi_k$ can be computed at a rather low cost. We will further show that using $\Xi_k$ for analysis would lead to a much larger, constant choice of low-level step size $\beta_k=1$. To prove \prettyref{thm:gradientConvergencePre}, we follow the similar steps as \prettyref{thm:gradientConvergence} and only list necessary changes in this section.

\subsection{Low-Level Error Propagation}
\begin{lemma}
\label{lem:lowLevelErrorPre}
Assuming $M(v)$ commutes with $K(v)$ and \ref{ass:StiffnessParameter}, the following relationship holds for all $k\geq1$:
\small
\begin{align*}
\Xi_{k+1}\leq\frac{p+1}{p}(1-2\beta_k\frac{\underline{\rho}}{\bar{\rho}_M}+\beta_k^2\frac{\bar{\rho}^2}{\underline{\rho}_M^2})\Xi_k+
L_K^2L_{\nabla K}^2\alpha_k^2\frac{p+1}{4}\|u_k\|^6,
\end{align*}
\normalsize
where we define: $\|K(v_{k+1})-K(v_k)\|_F^2\leq L_K\|v_{k+1}-v_k\|^2$ for any $v_k,v_{k+1}\in X$.
\end{lemma}
\begin{proof}
The following result holds by the triangle inequality and the bounded spectrum of $K(v)$ (\prettyref{eq:KVBound}):
\begin{align*}
&\|K(v_k)u_{k+1}-f\|^2\\
=&\|K(v_k)(u_{k+1}-u_k)\|^2+\|K(v_k)u_k-f\|^2+\\
&2\left<K(v_k)(u_{k+1}-u_k),K(v_k)u_k-f\right>\\
=&\|K(v_k)(u_{k+1}-u_k)\|^2+\|K(v_k)u_k-f\|^2-\\
&2\beta_k\left<K(v_k)M^{-1}(v_k)(K(v_k)u_k-f),K(v_k)u_k-f\right>\\
\leq&\|K(v_k)(u_{k+1}-u_k)\|^2+(1-2\beta_k\frac{\underline{\rho}}{\bar{\rho}_M})\|K(v_k)u_k-f\|^2\\
\leq&(1-2\beta_k\frac{\underline{\rho}}{\bar{\rho}_M}+\beta_k^2\frac{\bar{\rho}^2}{\underline{\rho}_M^2})\|K(v_k)u_k-f\|^2.\numberthis\label{eq:lowLevelTrackingUpdatePre}
\end{align*}
Next, we estimate an update due to $v_k$ via the Young's inequality:
\begin{align*}
&\|(K(v_{k+1})-K(v_k))u_{k+1}\|\|K(v_k)u_{k+1}-f\|\\
\leq&\frac{1}{2p}\|K(v_k)u_{k+1}-f\|^2+\frac{p}{2}\|(K(v_{k+1})-K(v_k))u_{k+1}\|^2\\
\leq&\frac{1}{2p}\|K(v_k)u_{k+1}-f\|^2+\frac{pL_K^2\alpha_k^2}{8}\|u_k^T\FPP{K}{v_k}u_k\|^2\|u_{k+1}\|^2.
\end{align*}
Putting the above two equations together, we have:
\small
\begin{align*}
&\|K(v_{k+1})u_{k+1}-f\|^2\\
\leq&\|(K(v_{k+1})-K(v_k))u_{k+1}\|^2+\|K(v_k)u_{k+1}-f\|^2+\\
&2\|(K(v_{k+1})-K(v_k))u_{k+1}\|\|K(v_k)u_{k+1}-f\|\\
\leq&\frac{p+1}{p}(1-2\beta_k\frac{\underline{\rho}}{\bar{\rho}_M}+\beta_k^2\frac{\bar{\rho}^2}{\underline{\rho}_M^2})\|K(v_k)u_k-f\|^2+\\
&L_K^2\alpha_k^2\frac{p+1}{4}\|u_k^T\FPP{K}{v_k}u_k\|^2\|u_{k+1}\|^2,\numberthis\label{eq:lowLevelErrorPre}
\end{align*}
\normalsize
which is a recursive relationship to be used for proving the low-level convergence.
\end{proof}

\subsection{Low-Level Convergence}
We use the following shorthand notation for the result in \prettyref{lem:lowLevelErrorPre}:
\small
\begin{align*}
&\Xi_{k+1}\leq\Theta_k\Xi_k+\Gamma_k\|u_k\|^6   \\
&\Theta_k\triangleq\frac{p+1}{p}(1-2\beta_k\frac{\underline{\rho}}{\bar{\rho}_M}+\beta_k^2\frac{\bar{\rho}^2}{\underline{\rho}_M})\quad
\Gamma_k\triangleq L_K^2L_{\nabla K}^2\alpha_k^2\frac{p+1}{4}.\numberthis\label{eq:lowLevelErrorShortPre}
\end{align*}
\normalsize
By taking \prettyref{ass:BetaChoice} and direct calculation, we can ensure that $\Theta_k\leq\bar{\Theta}<1$ (i.e., the first term is contractive). To bound the growth of the second term above, we show by induction that both $\Xi_k$ and $u_k$ can be uniformly bounded for all $k$ via a sufficiently small, constant $\Gamma_k\leq\bar{\Gamma}$.
\begin{lemma}
\label{lem:UpperBoundPre}
Assuming $M(v)$ commutes with $K(v)$, \ref{ass:StiffnessParameter}, \ref{ass:BetaChoicePre}, \ref{ass:GammaChoicePre}, we have $\Xi_k\leq U^2,\|u_k\|\leq U_f+U$ for all $k\geq1$, where $U_f$ is the uniform upper bound for $\|u_f(v)\|$.
\end{lemma}
\begin{proof}
First, since $v\in X$ and $X$ is compact, we have $\|u_f(v)\|\leq U_f<\infty$. Next, we prove $\Xi_k\leq U^2$ by induction. We already have $\Xi_1\leq U^2$. Now suppose $\Xi_k\leq U^2$, then \prettyref{eq:lowLevelErrorShortPre} and our assumption on $\bar{\Gamma}$ immediately leads to:
\begin{align*}
\Xi_{k+1}\leq&\bar{\Theta}U^2+\bar{\Gamma}(\|K^{-1}(v_k)(K(v_k)u_k-f)\|+\|u_f(v_k)\|)^6\\
\leq&\bar{\Theta}U^2+\bar{\Gamma}(U/\underline{\rho}+U_f)^6\leq U^2.
\end{align*}
Finally, we have: $\|u_k\|\leq\|u_k-u_f(v_k)\|+\|u_f(v_k)\|\leq U+U_f$ and our lemma follows.
\end{proof}
The shrinking coefficient $\bar{\Theta}$ and the uniform boundedness of $\Xi_k,u_k$ allows us to establish low-level convergence.
\begin{theorem}
\label{thm:lowLevelConvergencePre}
Taking \prettyref{ass:StiffnessParameter}, \ref{ass:BetaChoicePre}, \ref{ass:GammaChoicePre}, \ref{ass:AlphaChoicePre}, we can upper bound $\Xi_{k+1}$ as:
\small
\begin{align*}
\Xi_{k+1}\leq\bar{\Theta}^{k-1}\Xi_1+(U+U_f)^6\Gamma_1\sum_{i=0}^{k-1}\frac{\bar{\Theta}^i}{(k-i)^{2m }}=\mathcal{O}(k^{1-2m}).
\end{align*}
\normalsize
\end{theorem}
\begin{proof}
Recursively expand on \prettyref{eq:lowLevelErrorShort} and we have:
\small
\begin{align*}
&\Xi_{k+1}
\leq\bar{\Theta}^k\Xi_1+(U+U_f)^6\sum_{i=0}^{k-1}\bar{\Theta}^i\Gamma_{k-i}\\
\leq&\bar{\Theta}^k\Xi_1+(U+U_f)^6\Gamma_1\left[\frac{\lceil\frac{k+1}{2}\rceil}{(k-\lceil\frac{k-1}{2}\rceil)^{2m}}+\frac{\bar{\Theta}^{\lceil\frac{k}{2}\rceil}}{1-\bar{\Theta}}\right]
=\mathcal{O}(k^{1-2m}),
\end{align*}
\normalsize
where we have used our choice of $\alpha_k$, \prettyref{lem:UpperBoundPre}, and a similar argument as in \prettyref{thm:lowLevelConvergence}.
\end{proof}
\prettyref{thm:lowLevelConvergencePre} allows us to choose $m$ and tune the convergence speed of the low-level problem, which is used to establish the convergence of high-level problem.
\subsection{High-Level Error Propagation}
We first establish the high-level rule of error propagation over a single iteration:
\begin{lemma}
\label{lem:highLevelErrorPre}
Assuming $M(v)$ commutes with $K(v)$, \ref{ass:StiffnessParameter}, \ref{ass:BetaChoice}, \ref{ass:GammaChoice}, \ref{ass:AlphaChoice}, the following high-level error propagation rule holds for all $k\geq1,q>0$:
\begin{align*}
&l(v_{k+1})-l(v_k)\leq-\frac{1}{\alpha_k}\|v_{k+1}-v_k\|^2+\\
&\frac{L_{\nabla K}^2\alpha_k^2(L_{\nabla l}q+1)}{8q}(U+U_f)^4+\frac{L_{\nabla K}^2q}{8\underline{\rho}}\Xi_k(U+2U_f)^2.
\end{align*}
\end{lemma}
\begin{proof}
\prettyref{eq:highLevelConvexity} holds by the same argument as in \prettyref{lem:highLevelError}. For the last term in \prettyref{eq:highLevelConvexity}, we have:
\small
\begin{align*}
&\left<v_{k+1}-v_k,\frac{1}{2}u_k^T\FPP{K}{v_k}u_k-\frac{1}{2}u_f^T(v_k)\FPP{K}{v_k}u_f(v_k)\right>\\
\leq&\frac{L_{\nabla K}^2\alpha_k^2}{8q}\|u_k\|^4+
\frac{q}{2}\|\frac{1}{2}(u_k-u_f(v_k))^T\FPP{K}{v_k}(u_k+u_f(v_k))\|^2\\
\leq&\frac{L_{\nabla K}^2\alpha_k^2}{8q}\|u_k\|^4+\frac{L_{\nabla K}^2q}{8\underline{\rho}}\Xi_k\|u_k+u_f(v_k)\|^2.\numberthis\label{eq:highLevelYoungPre}
\end{align*}
\normalsize
The lemma follows by combining \prettyref{eq:highLevelConvexity}, \prettyref{eq:highLevelYoungPre}, and \prettyref{lem:UpperBoundPre}.
\end{proof}

\subsection{High-Level Convergence}
We first show that $\Xi_k^v$ is diminishing via the follow lemma:
\begin{lemma}
\label{lem:highLevelConvergencePre}
Assuming $M(v)$ commutes with $K(v)$, \ref{ass:StiffnessParameter}, \ref{ass:BetaChoicePre}, \ref{ass:GammaChoicePre}, \ref{ass:AlphaChoicePre}, we have the following bound on the accumulated high-level error $\Delta_k^v$:
\begin{align*}
&\frac{1}{2}\sum_{i=1}^k\alpha_i\Delta_i^v\leq l(v_1)-\underline{l}+\\
&\frac{L_{\nabla K}^2}{8}(U+U_f)^4\sum_{i=1}^k\alpha_i^{2-n}(L_{\nabla l}\alpha_i^n+1)+\\
&\frac{L_{\nabla K}^2}{8\underline{\rho}}(U+2U_f)^2\sum_{i=1}^k\alpha_i^n\Xi_i+
\frac{L_{\nabla K}^2}{4\underline{\rho}}(U+2U_f)^2\sum_{i=1}^k\alpha_i\Xi_i,
\end{align*}
where $\underline{l}$ is the lower bound of $l$ on $X$.
\end{lemma}
\begin{proof}
By choosing $q=\alpha_k^n$ for some constant $n$ and summing up the recursive rule of \prettyref{lem:highLevelErrorPre}, we have the following result:
\begin{align*}
&l(v_{k+1})-l(v_1)\leq-\sum_{i=1}^k\frac{1}{\alpha_i}\|v_{i+1}-v_i\|^2+\\
&\frac{L_{\nabla K}^2}{8}(U+U_f)^4\sum_{i=1}^k\alpha_i^{2-n}(L_{\nabla l}\alpha_i^n+1)+\\
&\frac{L_{\nabla K}^2}{8\underline{\rho}}(U+2U_f)^2\sum_{i=1}^k\alpha_i^n\Xi_i.\numberthis\label{eq:highLevelAccumulationPre}
\end{align*}
We further estimate the update using true gradient:
\small
\begin{align*}
&\Delta_k^v\leq\frac{2}{\alpha_k^2}\|v_{k+1}-v_k\|^2+\frac{L_{\nabla K}^2}{2\underline{\rho}}\Xi_k(U+2U_f)^2.\numberthis\label{eq:highLevelTrueErrorPre}
\end{align*}
\normalsize
We can prove our lemma by combining \prettyref{lem:highLevelErrorPre}, \prettyref{eq:highLevelAccumulationPre}, and \prettyref{eq:highLevelTrueErrorPre}.
\end{proof}
The following corollary can be proved by the same argument as \prettyref{cor:projectionError}:
\begin{corollary}
\label{cor:projectionErrorPre}
Assuming $M(v)$ commutes with $K(v)$, \ref{ass:StiffnessParameter}, \ref{ass:BetaChoicePre}, \ref{ass:GammaChoicePre}, \ref{ass:AlphaChoicePre}, \ref{ass:MNChoice}, we have $\min_{i=1,\cdots,k}\alpha_i\Xi_i^v\leq Ck^{-1}$ and there are infinitely many $k$ such that $\Xi_k^v\leq C_vk^{m-1}$ for some constants $C,C_v$ independent of $k$.
\end{corollary}
Finally, we list necessary changes to prove \prettyref{thm:gradientConvergencePre}:
\begin{proof}[Proof of \prettyref{thm:gradientConvergencePre}]
We choose $\tilde{v}_{k+1}, d_k, \tau_k$ as in proof of \prettyref{thm:gradientConvergence}. Applying \prettyref{cor:projectionErrorPre} and we can choose sufficiently large $k$ such that:
\begin{align*}
&-\left<\nabla l(\tilde{v}_{k+1}),\frac{\tau_kd_k}{\tau_k\|d_k\|}\right>\leq\frac{1}{\alpha_k}\|v_k-\tilde{v}_{k+1}\|\\
\leq&\frac{1}{\alpha_k}\|v_k-v_{k+1}\|+\frac{1}{\alpha_k}\|v_{k+1}-\tilde{v}_{k+1}\|\\
=&\sqrt{\Delta_k^v}+
\frac{1}{\alpha_k}\|v_{k+1}-\tilde{v}_{k+1}\|\\
\leq&\frac{\epsilon}{4}+\sqrt{\frac{L_{\nabla K}^2}{4\underline{\rho}}\Xi_k\|u_k+u_f(v_k)\|^2}
=\frac{\epsilon}{4}+\mathcal{O}(k^{(1-2m)/2}).
\end{align*}
The last inequality holds by the contractive property of projection operator, \prettyref{thm:lowLevelConvergencePre}, and \prettyref{lem:highLevelErrorPre}.
\end{proof}

\subsection{\label{sec:parameter}Parameter Choices}
We argue that $\beta_k=1$ is a valid choice in \prettyref{alg:CPOpt} if a Krylov-preconditioner is used. Note that choosing $\beta_k=1$ is generally incompatible with \prettyref{ass:BetaChoicePre}, which is because we make minimal assumption on the preconditioner $M(v)$, only requiring it to be commuting with $K(v)$ and positive definite with bounded spectrum. However, many practical preconditioners can provide much stronger guarantee leading to $\beta_k=1$ being a valid choice. To see this, we observe that the purpose of choosing $\beta_k$ according to \prettyref{ass:BetaChoicePre} is to establish the following contractive property in \prettyref{eq:lowLevelTrackingUpdatePre}:
\begin{align}
\label{eq:contractXi}
\|K(v_k)u_{k+1}-f\|^2\leq\bar{\Theta}\|K(v_k)u_k-f\|^2.
\end{align}
Any preconditioner can be used if they pertain such property with $\beta_k=1$. The following result shows that Krylov-preconditioner pertains the property:
\begin{lemma}
Suppose $M(v)$ is the Krylov-preconditioner with positive $D$ and $\beta_k=1$, then there exists some constant $\bar{\Theta}<1$ where \prettyref{eq:contractXi} holds in \prettyref{alg:CPOpt} for any $k$.
\end{lemma}
\begin{proof}
We define $b\triangleq K(v_k)u_k-f$ and, by the definition of $M(v)$, we have:
\begin{align*}
&\|K(v_k)u_{k+1}-f\|^2=\|K(v_k)(u_k-M^{-1}(v_k)b)-f\|^2\\
=&\|b-K(v_k)M^{-1}(v_k)b\|^2=\|b-\sum_{i=1}^{D+1}c_i^*K^i(v_k)b\|^2\\
\leq&\|(I-\frac{1}{\bar{\rho}}K(v_k))b\|^2\leq(1-\frac{\underline{\rho}}{\bar{\rho}})\|b\|^2\triangleq\bar{\Theta}\|K(v_k)u_k-f\|^2,
\end{align*}
where the first inequality holds by the definition of $c_i^*$.
\end{proof}
Unfortunately, it is very difficult to theoretically establish this property for other practical preconditioners, such as incomplete LU and multigrid, although it is almost always observed in practice.
\fi
\end{document}